%% file: main.tex
\documentclass{article}

\usepackage{microtype}
\usepackage{graphicx}
\usepackage{booktabs} 

\usepackage{hyperref}

\usepackage[accepted]{icml2024}

\usepackage{amsmath}
\usepackage{amssymb}
\usepackage{mathtools}
\usepackage{amsthm}
\usepackage{bm}
\usepackage{bbm}
\usepackage{wrapfig}
\usepackage{mathtools}
\usepackage{algorithmic}
\usepackage{subcaption}
\usepackage{enumitem}

\usepackage[capitalize,noabbrev]{cleveref}

\input{math_commands}

\theoremstyle{plain}

\theoremstyle{definition}

\theoremstyle{remark}

\usepackage[textsize=tiny]{todonotes}

\icmltitlerunning{Discrete Diffusion Modeling by Estimating the Ratios of the Data Distribution}

\begin{document}

\twocolumn[
\icmltitle{Discrete Diffusion Modeling by Estimating the Ratios of the Data Distribution}



\icmlsetsymbol{equal}{*}

\begin{icmlauthorlist}
\icmlauthor{Aaron Lou}{s}
\icmlauthor{Chenlin Meng}{s,p}
\icmlauthor{Stefano Ermon}{s}
\end{icmlauthorlist}

\icmlaffiliation{s}{Stanford University}
\icmlaffiliation{p}{Pika Labs}

\icmlcorrespondingauthor{Aaron Lou}{aaronlou@stanford.edu}

\icmlkeywords{Machine Learning, ICML}

\vskip 0.3in
]



\printAffiliationsAndNotice{}  

\begin{abstract}
    \input{content/00_abstract}
\end{abstract}

\input{content/01_introduction}
\input{content/02_preliminaries}
\input{content/03_method}
\input{content/04_simulation}
\input{content/05_experiments}
\input{content/06_related}

\input{content/07_limit_future}
\input{content/08_conclusion}

\nocite{langley00}

\bibliography{refs}
\bibliographystyle{icml2024}

\newpage
\appendix
\onecolumn

\input{content/A_proofs}
\input{content/B_training_details}
\input{content/C_experiments++}

\end{document}

%% file: math_commands.tex
\newtheorem{thm}{Theorem}[section]
\newtheorem{prop}[thm]{Proposition}
\newtheorem{defi}[thm]{Definition}

\newtheorem*{remark}{Remark}

\newcommand{\R}{\mathbb{R}}

\newcommand{\E}{\mathbb{E}}
\newcommand{\paren}[1]{\left(#1\right)}

\newcommand{\sqbrac}[1]{\left[#1\right]}

\newcommand{\grad}{\nabla}

\renewcommand{\vec}{\mathbf}

\newcommand\blfootnote[1]{%
  \begingroup
  \renewcommand\thefootnote{}\footnote{#1}%
  \addtocounter{footnote}{-1}%
  \endgroup
}

%% file: content/00_abstract.tex
Despite their groundbreaking performance for many generative modeling tasks, diffusion models have fallen short on discrete data domains such as natural language. Crucially, standard diffusion models rely on the well-established theory of score matching, but efforts to generalize this to discrete structures have not yielded the same empirical gains. In this work, we bridge this gap by proposing score entropy, a novel loss that naturally extends score matching to discrete spaces, integrates seamlessly to build discrete diffusion models, and significantly boosts performance. Experimentally, we test our Score Entropy Discrete Diffusion models (SEDD) on standard language modeling tasks. For comparable model sizes, SEDD beats existing language diffusion paradigms (reducing perplexity by $25$-$75$\%) and is competitive with autoregressive models, in particular outperforming GPT-2. Furthermore, compared to autoregressive mdoels, SEDD generates faithful text without requiring distribution annealing techniques like temperature scaling (around $6$-$8\times$ better generative perplexity than un-annealed GPT-2), can trade compute and quality (similar quality with $32\times$ fewer network evaluations), and enables controllable infilling (matching nucleus sampling quality while enabling other strategies besides left to right prompting).


%% file: content/01_introduction.tex
\section{Introduction}
Many recent advances in deep learning have centered around generative modeling. Here, a model learns how to generate novel samples from unstructured data. With the powerful capabilities of modern neural networks, these ``generative AI" systems have developed unparalleled capabilities, such as creating images given only text \citep{Ramesh2022HierarchicalTI} and answering complex questions \citep{Brown2020LanguageMA}.

The crucial part for any deep generative model is the probabilistic modeling technique. For discrete data such as natural language, autoregressive modeling \citep{yule1971method}--arguably the simplest modeling type since it derives from the probabilistic chain rule--has remained the only competitive method for decades. Although modern autoregressive transformers have produced stunning results \citep{Vaswani2017AttentionIA, Radford2019LanguageMA}, there are limits. For example, the sequential sampling of tokens is slow, hard to control, and often degrades without distribution annealing techniques like nucleus sampling \citep{holtzman2019curious}.

To alleviate these issues, researchers have sought alternative approaches to generating text data. In particular, inspired by their success in the image domain, many works have extended diffusion models \citep{SohlDickstein2015DeepUL, Ho2020DenoisingDP, Song2020ScoreBasedGM} to language domains \citep{Li2022DiffusionLMIC, Austin2021StructuredDD}. Yet, despite considerable effort, no such approach yet rivals autoregressive modeling, as they are not competitive on likelihoods, are slower to sample from, and do not generate comparable samples without resorting to heavy annealing and empirical alterations.

In our work, we challenge the longstanding dominance of autoregressive models by introducing Score Entropy Discrete Diffusion models (SEDD). SEDD parameterizes a reverse discrete diffusion process using the ratios of the data distribution. These are learned using score entropy, a novel loss that is analogous to score matching for standard diffusion models \citep{Hyvrinen2005EstimationON, Song2019GenerativeMB} and results in several empirical benefits\blfootnote{We open source our code at \href{https://github.com/louaaron/Score-Entropy-Discrete-Diffusion}{github.com/louaaron/Score-Entropy-Discrete-Diffusion}}:
\begin{enumerate}
    \item On core language modeling tasks, SEDD outperforms all existing language diffusion models \citep{Li2022DiffusionLMIC, Austin2021StructuredDD, Gulrajani2023LikelihoodBasedDL, He2022DiffusionBERTIG} by large margins and is competitive with autoregressive models of the same size (beating GPT-2 on its zero-shot perplexity tasks \citep{Radford2019LanguageMA}). 
    \item SEDD generates high quality unconditional samples and enables one to naturally trade off compute for quality. When measuring the generative perplexity (given by large models) of unconditional and un-annealed samples from similarly sized models, SEDD beats GPT-2 by $6$-$8\times$ and can match performance using $32\times$ fewer function evaluations.
    \item By directly parameterizing probability ratios, SEDD is highly controllable. In particular, one can prompt SEDD from arbitrary positions without specialized training. For both standard (left to right) and infilling, SEDD outperforms language diffusion models and is comparable with autoregressive models with nucleus sampling (as measured by MAUVE score \citep{pillutla2021mauve}).
\end{enumerate} 

    
    

%% file: content/02_preliminaries.tex
\section{Preliminaries}

\subsection{Discrete Diffusion Processes}

We will be modeling probability distributions over a finite support $\mathcal{X} = \{1, \dots, N\}$. As the support is discrete, note that our probability distributions can be represented by probability mass vectors $p \in \R^N$ that are positive and sum to $1$. To define a discrete diffusion process, we evolve a family of distributions $p_t \in \R^N$ according to the a continuous time Markov process given by a linear ordinary differential equation \citep{Campbell2022ACT, anderson2012continuous}:
\begin{equation}\label{eqn:discretediffusion}
    \frac{dp_t}{dt} = Q_t p_t \quad p_0 \approx p_{\rm data} 
\end{equation}
Here, $Q_t$ are the diffusion matrices $\R^{N \times N}$ and have non-negative non-diagonal entries and columns which sum to zero (so that the rate $\frac{dp_t}{dt}$ sums to $0$, meaning $p_t$ does not gain or lose total mass). Generally, $Q_t$ are simple (e.g. a simple scalar factor $Q_t = \sigma(t) Q$) so $p_t$ approaches a limiting distribution $p_{\rm base}$ as $t \to \infty$. 

One can simulate this process by taking small $\Delta t$ Euler steps and randomly sampling the resulting transitions. In particular, the samples are defined by transition densities which come from the columns of $Q_t$:
\begin{equation}\label{eqn:discrete_euler_base}
    p(x_{t + \Delta t} = y | x_t = x) = \delta_{xy} + Q_t(y, x) \Delta t + O(\Delta t^2)
\end{equation}

Finally, this process has a well known reversal \citep{Kelly1980ReversibilityAS, Sun2022ScorebasedCD} given by another diffusion matrix $\overline{Q}_t$:
\begin{multline}\label{eqn:reverse_discrete}
    \frac{dp_{T - t}}{dt} = \overline{Q}_{T - t} p_{T - t} \quad \overline{Q}_t(y, x) = \frac{p_t(y)}{p_t(x)} Q_t(x, y)\\
    \overline{Q}_t(x, x) = -\sum_{y \neq x} \overline{Q}_t(y, x)
\end{multline}

This reverse process is analogous to the time reversal for typical diffusion processes on $\R^n$, with the ratios $\frac{p_t(y)}{p_t(x)}$ (which are collectively known as the concrete score \citep{Meng2022ConcreteSM}) generalizing the typical score function $\grad_x \log p_t$ \citep{Song2019GenerativeMB} \footnote{The gradient operator for discrete structures is (up to some scaling) defined for pairs $x \neq y$ by $\grad f(xy) := f(y) - f(x)$. The score function would generalize to the normalized gradients $\frac{\grad p(xy)}{p(x)} = \frac{p(y)}{p(x)} - 1$.}

\subsection{Discrete Diffusion Models}

The goal of a discrete diffusion model is to construct the aforementioned reverse process by learning the ratios $\frac{p_t(y)}{p_t(x)}$. Unlike the continuous diffusion case, which has settled around (up to minor scaling variations) the theoretical framework given by score matching \citep{Hyvrinen2005EstimationON}, there currently exist many competing methods for learning discrete diffusion models. In particular, these tend to produce mixed empirical results, which spurs the need for a reexamination.

\textbf{Mean Prediction.} Instead of directly parameterizing the ratios $\frac{p_t(y)}{p_t(x)}$, \citet{Austin2021StructuredDD, Campbell2022ACT} instead follow a strategy of \citet{Ho2020DenoisingDP} to learn the reverse density $p_{0 | t}$. This actually recovers the ratios $\frac{p_t(y)}{p_t(x)}$ in a roundabout way (as shown in our Theorem \ref{thm:mean_param_score}), but comes with several drawbacks. First, learning $p_{0 | t}$ is inherently harder since it is a density (as opposed to a general value). Furthermore, the objective breaks down in continuous time and must be approximated \citep{Campbell2022ACT}. As a result, this framework largely underperforms empirically.

\textbf{Ratio Matching.} Originally introduced in \citet{Hyvrinen2007SomeEO} and augmented in \citet{Sun2022ScorebasedCD}, ratio matching learns the marginal probabilities of each dimension with maximum likelihood training. However, the resulting setup departs from standard score matching and requires specialized and expensive network architectures \citep{Chen2019NeuralNW}. As such, this tends to perform worse than mean prediction.

\textbf{Concrete Score Matching.} \citet{Meng2022ConcreteSM} generalizes the standard Fisher divergence in score matching, learning $s_\theta(x, t) \approx \begin{bmatrix}\frac{p_t(y)}{p_t(x)} \end{bmatrix}_{y \neq x}$ with concrete score matching:
\begin{equation}
    \mathcal{L}_{\rm CSM} = \frac{1}{2} \E_{x \sim p_t} \sqbrac{\sum_{y \neq x} \paren{s_\theta(x_t, t)_y - \frac{p_t(y)}{p_t(x)}}^2}
\end{equation}
Unfortunately, the $\ell^2$ loss is incompatible with the fact that $\frac{p_t(y)}{p_t(x)}$ must be positive. In particular, this does not sufficiently penalize negative or zero values, leading to divergent behavior. Although theoretically promising, Concrete Score Matching struggles (as seen in Appendix \ref{app:additional}).

%% file: content/03_method.tex
\section{Score Entropy Discrete Diffusion Models}

In this section, we introduce score entropy. Similar to concrete score matching, we learn the collected concrete score $s_\theta(x, t) \approx \begin{bmatrix}\frac{p_t(y)}{p_t(x)} \end{bmatrix}_{y \neq x}$ ($s_\theta: \mathcal{X} \times \R \to \R^{|\mathcal{X}}|$). We design the score entropy loss to incorporate the fact that these ratios are positive and evolve under a discrete diffusion.

\begin{defi}
    The \textbf{score entropy} $\mathcal{L}_{\rm SE}$ for a distribution $p$, weights $w_{xy} \ge 0$ and a score network $s_\theta(x)_y$ is
    
    \begin{equation}\label{eqn:score_entropy}
        \E_{x \sim p} \sqbrac{\sum_{y \neq x} w_{xy} \paren{s_\theta(x)_y - \frac{p(y)}{p(x)} \log s_\theta(x)_y + K\paren{\frac{p(y)}{p(x)}}}}
    \end{equation}
    where $K(a) = a (\log a - 1)$ is a normalizing constant function that ensures that $\mathcal{L}_{\rm SE} \ge 0$.
\end{defi}
\begin{remark}
    Instead of building off of Fisher divergences, score entropy builds off of the Bregman divergence $D_F\paren{s(x)_y, \frac{p(y)}{p(x)}}$ when $F = -\log$ is the convex function. As such, score entropy is non-negative, symmetric, and convex. It also generalizes standard cross entropy to general positive values (instead of simplex-valued probabilities), inspiring the name. The weights $w_{xy}$ are used primarily when combining score entropy with diffusion models.
\end{remark}
While this expression is more complex than the standard score matching variants, it satisfies several desiderata for a discrete diffusion training objective:

\subsection{Score Entropy Properties}


\underline{\textbf{First,}} score entropy is a suitable loss function that recovers the ground truth concrete score.

\begin{prop}[Consistency of Score Entropy]\label{prop:consist}
    Suppose $p$ is fully supported and $w_{xy} > 0$. As the number of samples and model capacity approaches $\infty$, the optimal $\theta^*$ that minimizes Equation \ref{eqn:score_entropy} satisfies $s_{\theta^*}(x)_y = \frac{p(y)}{p(x)}$ for all pairs $x, y$ Furthermore, $\mathcal{L}_{\rm SE}$ will be $0$ at $\theta^*$.
\end{prop}

\underline{\textbf{Second,}} score entropy directly improves upon concrete score matching by rescaling problematic gradients. For the weights $w_{xy} = 1$, $\grad_{s_\theta(x)_y} \mathcal{L}_{\rm SE} = \frac{1}{s_\theta(x)_y} \grad_{s_\theta(x)_y} \mathcal{L}_{\rm CSM}$, so the gradient signals for each pair $(x, y)$ are scaled by a factor of $s_\theta(x)_y$ as a normalization component. As such, this forms a natural log-barrier which keeps our $s_\theta \ge 0$.

\underline{\textbf{Third,}} similar to concrete score matching, score entropy can be made computationally tractable by removing the unknown $\frac{p(y)}{p(x)}$ term. There are two alternative forms, the first of which is analogous to the implicit score matching loss \citep{Hyvrinen2005EstimationON}:
    
\begin{prop}[Implicit Score Entropy]\label{prop:ise}
    $\mathcal{L}_{\rm SE}$ is equal up to a constant independent of $\theta$ to the \textbf{implicit score entropy}
    \begin{equation}
        \mathcal{L}_{\rm ISE} = \E_{x \sim p} \sqbrac{\sum_{y \neq x}w_{xy}  s_\theta(x)_y - w_{yx} \log s_\theta(y)_x}
    \end{equation}
\end{prop}
Unfortunately, a Monte Carlo estimate would require sampling an $x$ and evaluating $s_\theta(y)_x$ for all other $y$. For high dimensions, this is intractable, which means we have to sample $y$ uniformly, but this introduces additional variance analogous to that introduced by the Hutchinson trace estimator \citep{Hutchinson1989ASE} for sliced score matching \citep{Song2019SlicedSM}. As a result, implicit score entropy is impractical for large-scale tasks. Instead, we work a denoising score matching loss \citep{Vincent2011ACB} variant of score entropy:
\begin{thm}[Denoising Score Entropy]\label{thm:dse}
    Suppose $p$ is a perturbation of a base density $p_0$ by a transition kernel $p(\cdot | \cdot)$, ie $p(x) = \sum_{x_0} p(x | x_0) p_0(x_0)$. The score entropy $\mathcal{L}_{\rm SE}$ is equivalent (up to a constant independent of $\theta$) to the \textbf{denoising score entropy} $\mathcal{L}_{\rm DSE}$ is
    \begin{equation}
        \underset{\substack{x_0 \sim p_0 \\ x \sim p(\cdot | x_0)}}{\E} \sqbrac{\sum_{y \neq x} w_{xy} \paren{s_\theta(x)_y - \frac{p(y | x_0)}{p(x | x_0)} \log s_\theta(x)_y}}\\
    \end{equation}
\end{thm}
$\mathcal{L}_{\rm DSE}$ is scalable since Monte Carlo sampling only requires the evaluation of one $s_\theta(x)$, which gives us all $s_\theta(x)_y$, and the variance introduced by $x_0$ is manageable. Additionally, it is particularly appealing for discrete diffusion since the intermediate $p_t$ are all perturbations of the base density $p_0$ (resulting from Equations \ref{eqn:discretediffusion}, \ref{eqn:discrete_euler_base}), enabling us to train with $\mathcal{L}_{\rm DSE}$ using the diffusion transition densities $p_{t | 0}(\cdot | x_0)$ (which we can make tractable).

\subsection{Likelihood Bound For Score Entropy Discrete Diffusion}

\underline{\textbf{Fourth,}} the score entropy can be used to define an ELBO for likelihood-based training and evaluation.

\begin{defi}\label{def:discrete_diffusion_model}
    For our time dependent score network $s_\theta(\cdot, t)$, the parameterized reverse matrix is $\overline{Q}_t^\theta(y, x) = \begin{cases} s_\theta(x, t)_y Q_t(x, y) & x \neq y \\ -\sum_{z \neq x} \overline{Q}_t^\theta(z, y) & x = y \end{cases}$ found by replacing the ground truth scores in Equation \ref{eqn:reverse_discrete}. Our parameterized densities $p_t^\theta$ thus satisfy the following differential equation:
    \begin{equation}
        \frac{dp_{T - t}^\theta}{dt} = \overline{Q}_{T - t}^\theta p_{T - t}^\theta \quad p_T^\theta = p_{\rm base} \approx p_T
    \end{equation}
\end{defi}
The log likelihood of data points can be bounded using an ELBO based off of Dynkin's formula \citep{FloydASP}, which was derived for discrete diffusion models in \citet{Campbell2022ACT}. Interestingly, this takes the form of our denoising score entropy loss weighted by the forward diffusion:
\begin{thm}[Likelihood Training and Evaluation]\label{thm:likeli}
    For the diffusion and forward probabilities defined above,
    \begin{equation}
        -\log p_0^\theta(x_0) \le \mathcal{L}_{\rm DWDSE}(x_0) + D_{KL}(p_{T | 0}(\cdot | x_0) \parallel p_{\rm base})
    \end{equation}
    where $\mathcal{L}_{\rm DWDSE}(x_0)$ is the \textbf{diffusion weighted denoising score entropy} for data point $x_0$
    
    \begin{multline}
        \int_0^T \E_{x_t \sim p_{t | 0}(\cdot | x_0)} \sum_{y \neq x_t} Q_t(x_t, y) \Bigg(s_\theta(x_t, t)_y - \\
        \frac{p_{t | 0}(y | x_0)}{p_{t | 0}(x_t | x_0)} \log s_\theta(x_t, t)_y + K\paren{\frac{p_{t | 0}(y | x_0)}{p_{t | 0}(x_t | x_0)}}\Bigg) dt
    \end{multline}
\end{thm}
Crucially, this result allows us to directly models based on their likelihood values (and the related perplexity scores), the core metric for language modeling tasks. In particular, we can train and evaluate an upper bound.

\begin{remark}
    The DWDSE (and the implicit version) can be derived from the general framework of \citet{Benton2022FromDD} assuming a concrete score parameterization. In particular, the implicit version coincides with the likelihood loss introduced in \citet{Campbell2022ACT}.
\end{remark}

\subsection{Practical Implementation}

\underline{\textbf{Fifth,}} score entropy can be scaled to high dimensional tasks.

In practice, our state factorizes into sequences $\mathcal{X} = \{1, \dots, n\}^d$ to form sequences $\vec{x} = x^1 \dots x^d$ (e.g. sequences of tokens or image pixel values). As a general $Q_t$ would be of exponential size, we instead choose a sparse structured matrix that perturbs tokens independently with a matrix $Q_t^{\rm tok}$. In particular, the nonzero entries of $Q_t$ are given by
\begin{equation}
    \hspace{-0.2cm}Q_t(x^1 \dots x^i \dots x^d, x^1 \dots \widehat{x}^i \dots x^d) = Q_t^{\rm tok}(x^i, \widehat{x}^i)
\end{equation}
Since $\mathcal{L}_{\rm DWDSE}$ weights the loss by $Q_t(x, y)$, this token level transition $Q_t$ renders most ratios irrelevant. In particular, we only need to model all ratios between sequences with Hamming distnace $1$, so we can build our score network $s_\theta(\cdot, t): \{1, \dots, n\}^d \to \R^{d \times n}$ as a seq-to-seq map:
\begin{equation}
    (s_\theta(x^1 \dots x^i \dots x^d, t))_{i, \widehat{x}^i} \approx \frac{p_t(x^1 \dots \widehat{x}^i \dots x^d)}{p_t(x^1 \dots x^i \dots x^d)} 
\end{equation}
To fully compute $\mathcal{L}_{\rm DWDSE}$, we just need to calculate the forward transition $p_{t | 0}^{\rm seq}(\cdot | \cdot)$. Luckily, this decomposes as each token is perturbed independently:
\begin{equation}
    p_{t | 0}^{\rm seq}(\vec{\widehat{x}} | \vec{x}) = \prod_{i = 1}^d p_{t | 0}^{\rm tok}(\widehat{x}^i | x^i)
\end{equation}
For each $p_{t | 0}^{\rm tok}(\cdot | \cdot)$, we employ the previously discussed strategy and set $Q_t^{\rm tok} = \sigma(t) Q^{\rm tok}$ for a noise level $\sigma$ and a fixed transition $Q^{\rm tok}$. This avoids numerical integration as, if we define $\overline{\sigma}(t)$ as the cumulative noise $\int_0^t \sigma(s) ds$, we have:
\begin{gather}
    p_{t | 0}^{\rm tok}(\cdot | x) = x\text{-th column of } \exp\paren{\overline{\sigma}(t) Q^{\rm tok}}
\end{gather}
There are some practical consequences that render most $Q^{\rm tok}$ unusable for large scale experiments (e.g. for GPT-2 tasks, $n=50257$). In particular, one is not able to store all edge weights $Q_{\rm tok}(i, j)$ since this takes around $20$ GB of GPU memory and is extremely slow to access. Furthermore, one must be able to compute the columns $\exp(\overline{\sigma}(t) \cdot Q^{\rm tok})$ to get the transition ratios, but this must avoid matrix-matrix multiplication again can't be stored in memory.

To sidestep these issues, we follow prior work \citep{Austin2021StructuredDD, Campbell2022ACT} and use two standard matrices with special structures. They arise, respectively, from considering a fully connected graph structure and from introducing a MASK absorbing state (similar to the BERT language modeling paradigm \citep{Devlin2019BERTPO}):
\begin{gather}
    Q^{\rm uniform} = \begin{bmatrix} 1 - N & 1 & \cdots & 1\\ 1 & 1 - N & \cdots & 1\\ \vdots & \vdots & \ddots & \vdots \\ 1 & 1 & \cdots & 1 - N\end{bmatrix}\\
    Q^{\rm absorb} = \begin{bmatrix} -1 & 0 & \cdots & 0 & 0\\ 0 & -1 & \cdots & 0 & 0\\ \vdots & \vdots & \ddots & \vdots & \vdots \\ 0 & 0 & \cdots & -1 & 0\\ 1 & 1 & \cdots & 1 & 0\end{bmatrix}
\end{gather}
With such a structured $Q$, one can quickly and cheaply compute all values in $\mathcal{L}_{\rm DWDSE}$. As such, our training iteration is about as fast and uses a similar amount of memory as standard autoregressive training. In particular, our training algorithm is given in Algorithm \ref{alg:train}.

%% file: content/04_simulation.tex
\section{Simulating Reverse Diffusion with Concrete Scores}

Given our scores $s_\theta$, we now derive various strategies for simulating a path $\vec{x}_t = x_t^1 x_t^2 \dots x_t^d \sim p_t$ of the reverse diffusion process. Notably, the additional information that we gain from $s_\theta$ being an approximate ratio of $p_t$ can be used to enhance the sampling process.

\subsection{Time-Reversal Strategies}\label{sec:simulation:analytic}

To simulate the diffusion in Definition \ref{def:discrete_diffusion_model}, one may be tempted to use the Euler strategy from Equation \ref{eqn:discrete_euler_base}. However, as noted in \citet{Campbell2022ACT}, this is inefficient because the structure of $Q_t^{\rm seq}$ only allows one position to be modified per step. Instead, a natural alternative has been to use $\tau$-leaping \citep{Gillespie2001ApproximateAS}, which performs an Euler step at each position simultaneously. In particular, given a sequence $\vec{x}_t$, we construct $\vec{x}_{t - \Delta t}$ by sampling each token $x_{t - \Delta t}^i$ (independently) from the corresponding probability
\begin{equation}\label{eqn:tau_leap_euler}
    \delta_{x_t^i}(x_{t - \Delta t}^i) + \Delta t Q_t^{\rm tok}(x_t^i, x_{t - \Delta t}^i) s_\theta(\vec{x}_t, t)_{i, x_{t - \Delta t}^i}
\end{equation}
While $\tau$-leaping is a viable simulation strategy, it is agnostic to fact that our $s_\theta$ approximates the true concrete score. In particular, knowing all $\frac{p_t(y)}{p_t(x)}$ enables optimal denoising, analogous to Tweedie's theorem \citep{Efron2011TweediesFA}:
\begin{thm}[Discrete Tweedie's Theorem]\label{thm:tweedie}
    Suppose that $p_t$ follows the diffusion ODE $dp_t = Q p_t$. Then the true denoiser is given by
    
    \begin{equation}
        p_{0 | t}(x_0 | x_t)=\paren{\exp(-tQ) \begin{bmatrix} \frac{p_t(i))}{p_t(x_t)} \end{bmatrix}_{i = 1}^N}_{x_0} \exp(t Q)(x_t, x_0)
    \end{equation}
\end{thm}
Unfortunately, we do not know all of the ratios (only ratios between Hamming distance 1 sequences). However, we can use this intuition to build a Tweedie denoiser analogue of $\tau$-leaping. In particular, we replace the token transition probabilities (for $x_{t - \Delta t}^i$) with the values
\begin{gather}\label{eqn:tau_leap_analytic}
    \hspace{-0.5cm}\big(\exp(-\sigma_t^{\Delta t} Q) s_\theta(\vec{x}_t, t)_i\big)_{x_{t - \Delta t}^i} \exp(\sigma_t^{\Delta t} Q)(x_t^i, x_{t - \Delta t}^i)\\
    \text{where } \sigma_t^{\Delta t} = (\overline{\sigma}(t) - \overline{\sigma}(t - \Delta t))
\end{gather}
This generalizes the theorem but enforces the tau-leaping independence condition and, in fact, is optimal:
\begin{thm}[Tweedie $\tau$-leaping]\label{thm:mean_param_score}
    Let $p_{t - \Delta t | t}^{\rm tweedie}(\vec{x}_{t - \Delta t} | \vec{x}_t)$ be the probability of the token update rule defined by Equation \ref{eqn:tau_leap_analytic}. Assuming $s_\theta$ is learned perfectly, this minimizes the KL divergence with the true reverse $p_{t - \Delta t | t}(\vec{x}_{t - \Delta t} | \vec{x}_t)$ for all $\tau$-leaping strategies (i.e. token transitions are applied independently and simultaneously). 
\end{thm}

These simulation algorithms are unified in Algorithm \ref{alg:sample}.

\subsection{Arbitrary Prompting and Infilling}

Our concrete score can also be used to enable greater control over the generative process. This is due to the fact that we are modeling a function of the probability, allowing us to include conditional information through Bayes' rule. In particular, we consider the infilling problem
\begin{equation}
    p_t(\vec{x}^\Omega | \vec{x}^{\overline{\Omega}} = \vec{y}) \quad \Omega \text{ unfilled indices} \quad \overline{\Omega} \text{ filled}
\end{equation}
As an example, a standard autoregressive conditional generation would have $\overline{\Omega} = \{1, 2, \dots, c\}$ and $\Omega = \{c + 1, c + 2, \dots, d\}$. By Bayes' rule, the conditional scores can be recovered exactly from the unconditional score.
\begin{equation}
    \frac{p_t(\vec{x}^\Omega = \vec{z}' | \vec{x}^{\overline{\Omega}} = \vec{y})}{p_t(\vec{x}^\Omega = \vec{z} | \vec{x}^{\overline{\Omega}} = \vec{y})} = \frac{p_t(\vec{x} = \vec{z}' \oplus_\Omega \vec{y})}{p_t(\vec{x} = \vec{z} \oplus_\Omega \vec{y})}
\end{equation}
where $\oplus_\Omega$ is concatenation along $\Omega$ and $\overline{\Omega}$. Since the unconditional and conditional scores coincide, we can use our $s_\theta$ (learned unconditionally) for conditional sampling (given arbitrary $\overline{\Omega}$). For a $\tau$-leaping update rule (Equation \ref{eqn:tau_leap_euler} or \ref{eqn:tau_leap_analytic}), one would only modify by changing the values at $\Omega$. An explicit pseudocode of this is given in Algorithm \ref{alg:cond_sample}.


%% file: content/05_experiments.tex
\section{Experiments}

We now empirically validate that our score entropy discrete diffusion (SEDD) model on a variety of language modeling tasks. We measure both perplexity (i.e. likelihood estimation capabilities) as well as generation quality, finding that our method performs quite well in both aspects.

\subsection{Model and Training Setup}

\begin{table*}[t]
    \centering
    \begin{tabular}{l|l|ccccc}
    Size & Model & LAMBADA & WikiText2 & PTB & WikiText103 & 1BW \\ \hline
    Small & \multicolumn{1}{l|}{GPT-2} & \textbf{45.04} & 42.43 & 138.43 & 41.60 & \textbf{75.20}\\
    & \multicolumn{1}{l|}{SEDD Absorb} & $\le$50.92 & $\le$\textbf{41.84} & $\le$\textbf{114.24} & $\le$\textbf{40.62} & $\le$79.29\\
    & \multicolumn{1}{l|}{SEDD Uniform} & $\le$65.40 & $\le$50.27 & $\le$140.12 & $\le$49.60 & $\le$101.37\\ 
    & \multicolumn{1}{l|}{D3PM} & $\le$93.47 & $\le$77.28 & $\le$200.82 & $\le$75.16 & $\le$138.92\\
    & \multicolumn{1}{l|}{PLAID} & $\le$57.28 & $\le$51.80 & $\le$142.60 & $\le$50.86 & $\le$91.12\\ \hline
    Medium & \multicolumn{1}{l|}{GPT-2} & \textbf{35.66} & 31.80 & 123.14 & 31.39 & \textbf{55.72}\\
    & \multicolumn{1}{l|}{SEDD Absorb} & $\le$42.77 & $\le$\textbf{31.04} & $\le$\textbf{87.12} & $\le$\textbf{29.98} & $\le$61.19\\
    & \multicolumn{1}{l|}{SEDD Uniform} & $\le$51.28 & $\le$38.93 & $\le$102.28 & $\le$36.81 & $\le$79.12\\
    \end{tabular}
    \caption{\textbf{Zero-shot unconditional perplexity ($\downarrow$) on a variety of datasets.} For a fixed size, the best perplexity is \textbf{bolded}. Our SEDD model with absorbing transition beats GPT-2 \citep{Radford2019LanguageMA} on a majority of the tasks and entirely outperforms prior language diffusion models \citep{Austin2021StructuredDD, Gulrajani2023LikelihoodBasedDL}.}
    \label{tbl:perplexity}
\end{table*}

Our core model is based on the diffusion transformer architecture \citep{Peebles2022ScalableDM}, which incorporates time conditioning into a standard encoder-only transformer architecture \citep{Vaswani2017AttentionIA, Devlin2019BERTPO}, although we make some minor modifications such as employing rotary positional encoding \citep{Su2021RoFormerET}.

We construct SEDD Absorb and SEDD Uniform, which correspond to the matrices $Q^{\rm uniform}$ and $Q^{\rm absorb}$ respectively. We tested a geometric noise schedule (that interpolates between $10^{-5}$ and $20$), as well as a log-linear noise schedule (the number of changed tokens for total noise $\overline{\sigma}(t)$ is approximately $td$ for both transitions), which helps SEDD Absorb for perplexities. Outside of this, we did not systemically explore noise schedules or alternative loss weightings, although these could likely improve generation quality.

When training, we employ sentence packing to create uniform length blocks to feed to our model, which is done typically for language modeling tasks. The only exception to this rule is our experiment on text8, which randomly samples contiguous subsequences to match prior work \citep{Austin2021StructuredDD} (although we found that this did not substantially change results). We also matched architecture hyperparameters with prior work (including number of layers, hidden dimension, attention heads, etc...), although our models have slightly more parameters ($\approx 5-10\%$) than a typical transformer due to time conditioning. We also use the same tokenizers as prior work (which otherwise could be a source of artifacts) as well as the same data splits.

\subsection{Language Modeling Comparison}

We begin by evaluating our model on core language modeling (effectively likelihood-based modeling) on three common datasets across a variety of scales.

\subsubsection{Text 8 Dataset}

We compare on the text8 dataset, a small, character level language modeling task. We follow \citet{Austin2021StructuredDD} for network hyperparameters and dataset splits and compare with methods that employ a similar model size. 

We report bits per character (BPC) in Table \ref{tbl:text8}. SEDD outperforms other non-autoregressive models and is only beaten by an autoregressive transformer and the discrete flow (which incorporates an autoregressive base distribution) \citep{tran2019discrete}. Furthermore, SEDD substantially improves upon D3PM \citep{Austin2021StructuredDD}, despite both being built from the same discrete diffusion principles. 

\subsubsection{One Billion Words Dataset}

\begin{table}[]
    \centering
    \begin{tabular}{l|l|r}
    Type & Method &  BPC ($\downarrow$) \\ \hline
    Autoregressive Backbone& IAF/SCF & 1.88\\
     & AR Argmax Flow & 1.39\\
    & Discrete Flow & \textbf{1.23}\\
    & Autoregressive & \textbf{1.23}\\ \hline
    Non-autoregressive & Mult. Diffusion & $\le$ 1.72\\
    & MAC & $\le$ 1.40\\
    & BFN & $\le$ 1.41\\
    & D3PM Uniform & $\le$ 1.61\\
    & D3PM Absorb & $\le$ 1.45\\ \hline
    Ours (NAR) & SEDD Uniform & $\le$ 1.47\\
    & SEDD Absorb & $\le$ \textbf{1.39}\\
    \end{tabular}
\caption{\textbf{Bits Per Character on text8.} Our SEDD models achieve second-best overall result (best for non-autoregressive), only being beaten out by the autoregressive model and a discrete flow (which uses an autoregressive model as a backbone) by a small margin. SEDD also substantially improves upon prior the discrete diffusion model D3PM \citep{Austin2021StructuredDD}.}
\label{tbl:text8}
\vspace{-0.5cm}
\end{table}

We also test SEDD on One Billion Words, a more medium sized and real world dataset. We follow \citet{He2022DiffusionBERTIG} for the tokenization, training, and model size configurations. In particular, our baselines are all around the size of GPT-2 small. Following \citet{He2022DiffusionBERTIG}, we compare primarily against other language diffusion models, although we also train a standard autoregressive transformer as a benchmark.

We report perplexity values in Table \ref{tbl:1bw}. Our SEDD model outperforms all other diffusion language modeling schemes by $50$-$75\%$ lower perplexity (in particular D3PM). Furthermore, SEDD is within $1$ perplexity of the autoregressive model, likely matching since we only report an upper bound.

\begin{table}[]
    \centering
    \begin{tabular}{l|l|r}
    Type & Method & Perplexity ($\downarrow$)\\ \hline
    Autoregressive & Transformer & \textbf{31.98}\\ \hline
    Diffusion & D3PM Absorb & $\le$ 77.50\\
    & Diffusion-LM & $\le$ 118.62 \\
    & BERT-Mouth & $\le$ 142.89\\
    & DiffusionBert & $\le$ 63.78\\ \hline
    Ours (Diffusion) & SEDD Uniform & $\le$ 40.25 \\
    & SEDD Absorb & $\le$ \textbf{32.79} \\
    \end{tabular}
\caption{\textbf{Test perplexities on the One Billion Words Dataset.} The autoregressive result is an exact likelihood, while the diffusion results are upper bounds. SEDD beats all other discrete diffusion models (by at least $2\times$) while matching the autoregressive baseline.}
\label{tbl:1bw}
\vspace{-15pt}
\end{table}

\begin{figure*}[t]
  \centering

  \begin{subfigure}[b]{0.49\textwidth}
    \includegraphics[width=\textwidth]{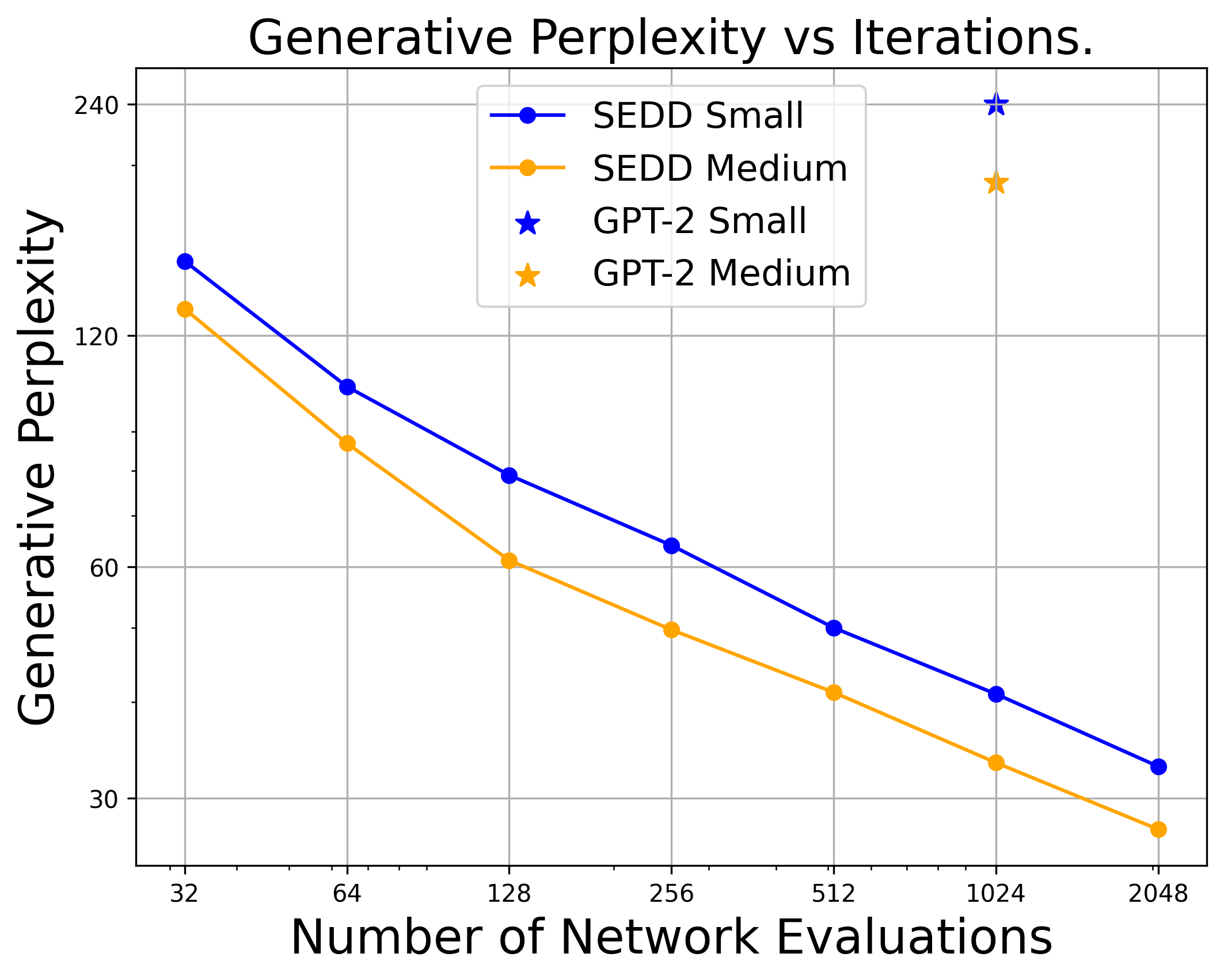}
    \caption{Generative Perplexity $(\downarrow)$ vs. Sampling Iterations.}
    \label{fig:gen_perpl_iterations}
  \end{subfigure}
  \hfill
  \begin{subfigure}[b]{0.47\textwidth}
        \begin{tabular}{|c| p{6.75cm}|}
            \hline
            \rotatebox[origin=r]{90}{\hspace{2pt} GPT-2 S} & {\fontfamily{lmr}\selectfont a hiring platform that "includes a fun club meeting place," says petitioner’s AQQFredericks. They’s the adjacent marijuana-hop. Others have allowed 3B Entertainment} \\
            \hline
            \rotatebox[origin=r]{90}{\hspace{2pt} GPT-2 M} & {\fontfamily{lmr}\selectfont  misused, whether via Uber, a higher-order reality of quantified impulse or the No Mass Paralysis movement, but the most shamefully universal example is gridlock} \\
            \hline
            \rotatebox[origin=r]{90}{\hspace{2pt} SEDD S\hspace{0.5pt}} & {\fontfamily{lmr}\selectfont As Jeff Romer recently wrote, ``The economy has now reached a corner - 64\% of household wealth and 80\% of wealth goes to credit cards because of government austerity} \\
            \hline
            \rotatebox[origin=r]{90}{\hspace{2pt} SEDD M \hspace{0pt}} & {\fontfamily{lmr}\selectfont Wyman worked as a computer science coach before going to work with the U.S. Secret Service in upstate New York in 2010. Without a license, the Secret Service will have to} \\
            \hline
        \end{tabular}
        \caption{Generated Text (small models)}
    \label{fig:subfig2}
  \end{subfigure}
  \caption{\textbf{Quality evaluation of unconditionally generated text.} We compare SEDD and GPT-2 by the perplexity of their analytically generated sequences. Our SEDD models consistently outperform GPT-2, interpolating between a $32 \times$ speedup and a $6$-$8\times$ improvement based on the chosen step size. The generated text reflects this improved generation capability, as our samples are far more coherent. Additional samples and ablations can be found in Appendix \ref{app:add:samples}}
  \label{fig:sample_quality}
\end{figure*}

\subsubsection{GPT-2 Zero Shot Tasks}

Finally, we compare SEDD against GPT-2 \citep{Radford2019LanguageMA}. We train on OpenWebText as the original WebText dataset has not been made available (this is typical practice and does not meaningfully affect results in practice) \citep{Gokaslan2019OpenWeb} and test on the LAMBADA, WikiText2, PTB, WikiText103, and One Billion Words datasets (which were all of the GPT-2 zero-shot tasks that measured perplexity). We recompute baseline likelihoods for all datasets except 1BW, where we encountered unexpected behavior with the public implementations. Our likelihood computation changes from the original setting since we evaluate unconditionally (i.e. without a sliding window), and this results in higher values than originally reported.

Our results are reported in Table \ref{tbl:perplexity}. Our SEDD Absorb beats GPT-2 on a majority of the zero-shot tasks across both sizes. To the best of our knowledge, this is the first time where a non-autoregressive language model has matched a modern, reasonably sized, and well-known autoregressive model for perplexities. We also compare against the most competitive continuous \citep{Gulrajani2023LikelihoodBasedDL} and discrete \citep{Austin2021StructuredDD} diffusion baselines, seeing a large improvement over both.

\subsection{Language Generation Comparison}

With our trained models, we compare against prior work in terms of generation quality. In particular, we compare GPT-2 with our SEDD Absorb on a variety of scales. Results for SEDD Uniform are given in Appendix \ref{app:additional}.

\subsubsection{Unconditional Generation}

We first compare the quality of unconditional samples between GPT-2 and SEDD. As most language metrics are meant for comparing conditional generations \citep{pillutla2021mauve}, we instead measure the generative perplexity of sampled sequences (using a GPT-2 large model for evaluation). This is a simple and common metric \citep{han2022ssd, Dieleman2022ContinuousDF} but can easily be ``hacked" by simple distribution annealing methods. So, we compare analytically sampled generations (i.e. no temperature scaling).

For SEDD, we simulate using 32 to 2048 steps, which approximates the learned distribution with minimal error for a large number of steps (the sequences are length 1024). Our results (both the measured generative perplexity and some samples) are shown in Figure \ref{fig:sample_quality}. SEDD matches GPT-2 quality using 32$\times$ fewer network evaluations and outperforms by $6$-$8\times$ when using the full 2048 steps. Furthermore, SEDD forms a predictable log-log linear pareto frontier between the number of sampling steps and generative perplexity. However, each network evaluation is different due to the KV-cache, which introduces a cost benefit tradeoff that we discuss more in Section \ref{sec:related}.

\begin{table*}[t]
    \centering
    \begin{tabular}{|p{17cm}|}
        \hline
            {\fontfamily{lmr}\selectfont \textcolor{blue}{A bow and arrow is a traditional weapon that enables an attacker to} attack targets at a range within a meter or maybe two meters. They have a range far longer than a human can walk, and they can be fired \dots} \\
            \hline
            {\fontfamily{lmr}\selectfont $\dots$ \textcolor{blue}{skydiving is a fun sport} that makes me feel incredibly silly. I think I may’ve spent too much, but it could’ve been amazing! While sky diving gives us exercise and fun, \textcolor{blue}{scuba diving} is an act of physical fitness, \dots} \\
            \hline
             {\fontfamily{lmr}\selectfont $\dots$ no one expected the results to much better than last year's one-sided endorsement. Nearly 90 percent of the results were surveyed as "independent," an promising \textcolor{blue}{result for school children across the country.}} \\
            \hline
            {\fontfamily{lmr}\selectfont $\dots$ results show that \textcolor{blue}{Donald Trump and Hillary Clinton} are in 38 states combined with less than 1\% of the national vote. In a way, it’s Trump and Hillary Clinton who will work overtime to get people to vote this $\dots$} \\
            \hline
        \end{tabular}
        \caption{\textbf{Conditionally Generated Text.} Prompt tokens are given in \textcolor{blue}{blue}. Our model is able to generate meaningful text with prompt tokens in the front, the end, the middle, or even split up. Additional samples are given in Appendix \ref{app:add:samples}.}
        \label{tbl:gen_text_infill}
\end{table*} 

\subsubsection{Infilling Conditional Generation}

Finally, we showcase SEDD's ability for conditional generation. We generate samples conditioned on a fixed amount of input text (from the WebText dataset) and compare their MAUVE scores \citep{pillutla2021mauve}. For SEDD, we consider two prompting strategies: standard generation given the beginning and infilling using the beginning and end, although obviously more sampling strategies exist (and several are visualized in Table \ref{tbl:gen_text_infill}).

We compare against GPT-2 and SSD-LM \citep{han2022ssd}, a competitive language diffusion model built for this task (all models are medium sized). Interestingly, a critical component for both baselines is distribution annealing: nucleus sampling for autoregressive modeling \citep{holtzman2019curious} (which clips the token probability) and thresholding for diffusion \citep{Li2022DiffusionLMIC, lou2023reflected} (which constrains generation to disallow paths in low probability spaces). As introducing similar annealing methods for SEDD is out of scope for this paper, we compare against both the annealed and un-annealed baselines samples.

Our results are given in Table \ref{tbl:conditional}. SEDD is highly competitive with the best configuration for both baselines, in fact beating both when using standard prompting. This is rather notable since SEDD does not use distribution annealing and does not explicitly encode left to right prompting as an architectural inductive bias (while GPT-2 and SSD-LM were trained explicitly for autoregressive-like generation).

\begin{table}[]
    \centering
    \begin{tabular}{l|l|l}
    \hline
    Method & Annealing & Mauve ($\uparrow$) \\ \hline
    GPT-2 & Nucleus-0.95  & 0.955 \\
    & None & 0.802 \\ 
    SSD-LM & Logit Threshold-0.95 & 0.919 \\ 
    & None & 0.312 \\ \hline
    SEDD Standard & None & \textbf{0.957}\\
    SEDD Infill & None & 0.942
    \end{tabular}
    \caption{\textbf{Evaluation of conditionally generated text.} SEDD with standard prompting beats both GPT-2 and SSD-LM. SEDD also offers more flexibility (enabling infilling generation with comparable performance) and does not require distribution annealing techniques for good generation.}
    \label{tbl:conditional}
    \vspace{-0.5cm}
\end{table}



%% file: content/06_related.tex
\section{Related Work}\label{sec:related}

\textbf{Continuous Diffusion Models for Text Data.} Initially proposed by \citet{Li2022DiffusionLMIC}, continuous language diffusion models embed tokens in a latent space, learn a diffusion model there, and take the nearest neighbor to dequantize. While initial versions struggled, these models have achieved significant results by iterating on several empirical components. For example, prior works improve downstream performance with alternative loss functions (moving away from likelihood-based score matching) \citep{han2022ssd, mahabadi2023tess} and explicitly encoding conditional information (e.g. inputting an infilling mask) \citep{gong2023diffuseq, Dieleman2022ContinuousDF}. Additionally, distribution annealing methods like thresholding \citep{Li2022DiffusionLMIC} and classifier-free guidance \citep{Ho2022ClassifierFreeDG} can further improve generation quality, although recent work has shown that methods like self-conditioning \citep{strudel2022self} and designing a less sparse embedding space (e.g. based on bits) \citep{chen2022analog} can obviate the need for such methods. Finally, \citet{Gulrajani2023LikelihoodBasedDL} showed that, with many surgical changes to the training paradigm, it is possible for language diffusion models to begin approaching autoregressive performance for likelihoods.

\textbf{Discrete Diffusion Models.} Most discrete diffusion works follow the framework set out by D3PM \citep{Austin2021StructuredDD} which mimics ``mean prediction" \citep{Ho2020DenoisingDP}. These discrete diffusion methods are largely applied to fields other than language (e.g. images), likely due to empirical challenges. Despite this, some works have shown strong performance on language, particularly for seq-to-seq tasks and more efficient generation \citep{Zheng2023ARD, chen2023fast, ye2023dinoiser}. Notably, from these works discrete diffusion has tended to be advantageous over continuous diffusion in reducing network evaluations.

\textbf{SEDD vs Prior Work.} SEDD is a discrete diffusion model that focuses on score matching, the crucial ingredient for continuous diffusions \citep{Song2019GenerativeMB,Ho2020DenoisingDP}. Many such works also focus on reversing a discrete diffusion process \citep{Campbell2022ACT, Benton2022FromDD, Sun2022ScorebasedCD}, so score entropy is naturally related with prior training objectives. However, SEDD focuses on a principled, scalable, and performant objective (namely denoising score entropy), filling in shortcomings found in previous works. In particular, prior methods train either with the equivalent of implicit score entropy (which is intractable and high variance) or propose alternate losses that suffer from other issues. These critical differences enable large improvements for language tasks, where prior discrete diffusion models have conspicuously struggled on.

Furthermore, SEDD achieves better results (for both perplexity and generation) than even continuous diffusion models (without resorting to empirically driven heuristics). This is desirable since discrete data should necessitate a novel approach. Future work could adapt empirical designs from continuous diffusion, further improving performance.

Finally, SEDD challenges autoregressive models, achieving competitive perplexities (beating GPT-2) and generation quality (beating nucleus sampling). While there is still a large gap with modern large language models, we believe that future work can bridge this using SEDD as a backbone. 

\textbf{SEDD vs Autoregressive Sampling Iterations.} SEDD and autoregressive models have significantly different sampling procedures due to the introduction of the KV-cache for standard decoder-only transformer models. In particular, this complicates the inference code (as each network pass changes from being a standard full batch forward) and trades off speed with memory. For example, for our (known) unoptimized codebase and the existing huggingface transformers library \citep{wolf-etal-2020-transformers}, we observed that SEDD matches autoregressive inference time when using around 100 steps but can increase the batch size by roughly $4-6$ times by removing the KV-cache memory. Future work will likely decrease the steps required for optimal generation (similar to existing work in standard diffusion \citep{song2021denoising}) which can improve this tradeoff.

%% file: content/08_conclusion.tex
\section{Conclusion}

We have introduced score entropy discrete diffusion (SEDD) models, a discrete diffusion model that is parameterized by the concrete score and can be trained efficiently with our novel score entropy loss. SEDD beats previous language diffusion models and rivals autoregressive models for both perplexity and quality. We hope that future work can build off our framework to defines alternatives to the modern autoregressive language modeling paradigm.

\section*{Impact Statement}

This paper proposes work that advances the field of natural language generation. Outside of existing ethical questions for this area (e.g. bias, toxicity, fake content), our approach does not present any specific danger as the core work is largely theoretical and not at the scale to pose a specific problem.

\section*{Acknowledgements}

This project was supported by NSF (\#1651565), ARO (W911NF-21-1-0125), ONR (N00014-23-1-2159), CZ Biohub, a Stanford HAI GCP grant. AL is supported by a NSF Graduate Research Fellowship.

%% file: content/A_proofs.tex
\section{Proof of Main Results}

\begin{proof}[Proof of Prop \ref{prop:consist}]
    Given infinite samples, the loss becomes equivalent to minimizing
    \begin{equation}
        \min_\theta \sum_{x, y \neq x} p(x) w_{xy}\paren{s_\theta(x)_y - \frac{p(y)}{p(x)} \log s_\theta(x)_y}
    \end{equation}
    where we have removed constants not depending on $\theta$. This is minimized when 
    \begin{equation}
        s_\theta(x)_y - \frac{p(y)}{p(x)} \log s_\theta(x)_y
    \end{equation}
    is minimized for all $x, y$. Taking a derivative with respect to $s$ and setting to $0$, we see that this occurs when $s_\theta(x)_y = \frac{p(y)}{p(x)}$, which can be easily checked to be optimal as the function is convex as a function of $s$. One can check that the loss is $0$ at the minimum.
\end{proof}

\begin{proof}[Proof of Prop \ref{prop:ise}]
    The trick is the categorical equivalent of the divergence theorem. In particular, we have
    \begin{align*}
        \E_{x \sim p} \sum_{y \neq x} \frac{p(y)}{p(x)} f(x, y) &= \sum_{x, y: x \neq y} \frac{p(y)}{p(x)} p(x) f(x, y)\\
        &= \sum_{x, y: x \neq y} p(y) f(x, y)\\
        &= \E_{y \sim p} \sum_{x \neq y} f(x, y)\\
        &= \E_{x \sim p} \sum_{y \neq x} f(y, x)
    \end{align*}
    for abitrary $f$. By setting $f(x, y) = w_{xy} \log s_\theta(x)_y$, we get that
    \begin{align*}
        &\E_{x \sim p} \sqbrac{\sum_{y \neq x} w_{xy} \paren{s_\theta(x)_y - \frac{p(y)}{p(x)} \log s_\theta(x)_y + K\paren{\frac{p(y)}{p(x)}}}} \\
        &= \E_{x \sim p} \sqbrac{\sum_{y \neq x} w_{xy} s_\theta(x)_y - w_{yx} \log s_\theta(y)_x + w_{xy} K\paren{\frac{p(y)}{p(x)}}}
    \end{align*}
    which is the desired equivalent (as the last term does not depend on $\theta$).
\end{proof}

\begin{proof}[Proof of Thm \ref{thm:dse}]
    This is similar to the same denoising variant for concrete score matching. We just need to show that the $\log s_\theta(x_t)_y \frac{p_t(y)}{p_t(x)}$ marginalizes out, since everything else does not change or is a constant.
    \begin{align*}
        \E_{x \sim p} \sum_{y \neq x} f(x, y) \frac{p(y)}{p(x)} &= \sum_{y \neq x} f(x, y) p_t(y)\\
        &= \sum_{y \neq x} \sum_{x_0} f(x_t, y) p(y | x_0) p_0(x_0)\\
        &= \E_{x_0 \sim p_0} \sum_{y \neq x} f(x, y) \frac{p(y | x_0)}{p(x | x_0)} p(x | x_0)\\
        &= \E_{x_0 \sim p_0, x \sim p(\cdot | x_0)} \sum_{y \neq x} f(x, y) \frac{p(y | x_0)}{p(x | x_0)}
    \end{align*}
    Applying this to our loss when $f(x, y) = w_{xy} \log s_\theta(x)_y$ gives us
    \begin{align*}
        &\E_{x \sim p} \sqbrac{\sum_{y \neq x} w_{xy} \paren{s_\theta(x)_y - \frac{p(y)}{p(x)} \log s_\theta(x)_y + K\paren{\frac{p(y)}{p(x)}}}} \\
        &= \E_{x \sim p} \sqbrac{\sum_{y \neq x} w_{xy} \paren{s_\theta(x)_y + K\paren{\frac{p(y)}{p(x)}}}} - \E_{x_0 \sim p_0, x \sim p(\cdot | x_0)} \sqbrac{\sum_{y \neq x} \frac{p(y | x_0)}{p(x | x_0)} w_{xy} \log s_\theta(x)_y}\\
        &= \E_{x_0 \sim p_0, x \sim p(\cdot | x_0)} \sqbrac{w_{xy} \paren{s_\theta(x)_y \frac{p(y | x_0)}{p(x | x_0)} \log s_\theta(x)_y + K\paren{\frac{p(y)}{p(x)}}}}
    \end{align*}
\end{proof}

\begin{proof}[Proof of Thm \ref{thm:likeli}]
    The full bound is given by

    \begin{equation}
        - \log p_0^\theta(x_0) \le \mathcal{L}_{\rm DWDSE}(x_0) + D_{\rm KL}(p_{T  | 0}(\cdot | x_0) \parallel \pi)
    \end{equation}
    where $\mathcal{L}_{\rm DWDSE}$ is given by
    \begin{equation*}
        \int_0^T \E_{x_t \sim p_{t | 0}(\cdot | x_0)} \sum_{y \neq x_t} Q_t(x_t, y) \paren{s_\theta(x_t, t)_y - \frac{p_{t | 0}(y | x_0)}{p_{t | 0}(x_t | x_0)} \log s_\theta(x, t)_y + K\paren{\frac{p_{t | 0}(y | x_0)}{p_{t | 0}(x_t | x_0)}}} dt
    \end{equation*}
    Effectively, $\mathcal{L}_{\rm DWSDE}$ is the path measure KL divergence \citep{Campbell2022ACT, Song2021MaximumLT}, and the proof follows similarly. In particular, we have that, by the data processing inequality
    \begin{equation}
        -\log p_0^\theta(x_0) = D_{\rm KL}(\delta_{x_0} \parallel p_0^\theta) \le D_{\rm KL}(\mathbb{P}_{x_0} \parallel \mathbb{P}^\theta)
    \end{equation}
    where $\mathbb{P}_{x_0}$ is the path measure for the reverse of the noising process applied to $\delta_{x_0}$ and $\mathbb{P}^\theta$ is the learned reverse process. Generally, we can replace $\delta_{x_0}$ with a more general data distribution $p_{\rm data}$, with the computation remaining the same. We have,
    \begin{equation}
        D_{\rm KL}(\mathbb{P}_{x_0} \parallel \mathbb{P}^\theta) \le \E_{x_T \sim p_{T | 0}(\cdot | x_0)} \sqbrac{D_{\rm KL}(\mathbb{P}_{x_0}(\cdot | x_T) \parallel \mathbb{P}^\theta(\cdot | x_T))} + D_{\rm KL}(p_{T  | 0}(\cdot | x_0) \parallel \pi)
    \end{equation}
    We analyze the term $\E_{x_T} D_{\rm KL}(\mathbb{P}_{x_0}(\cdot | x_T) \parallel \mathbb{P}^\theta(\cdot | x_T))$, which we can compute by Dynkin's formula \citep{FloydASP, Campbell2022ACT}, which, similar to Girsanov's Theorem for standard SDEs \citep{ksendal1987StochasticDE}, allows one to compute the change in measure. In particular, by applying Theorem 7.1 of \citet{FloydASP} with degenerate SDE coefficients, we find the expectation to be given explicitly by
    \begin{align}
        \int_0^T \E_{x_t \sim p_{t | 0}(\cdot | x_0)} &\sum_{y \neq x_t} \overline{Q}_t^\theta(y, x_t) - Q_t(y, x_t) \log(\overline{Q}_t^\theta(x_t, y))\\
        &+ Q_t(y, x_t) \log Q_t(y, x_t) + Q_t(x_t, y) K\paren{\frac{p_{t | 0}(y | x_0)}{p_{t | 0}(x_t | x_0)}}dt
    \end{align}
    Since our reverse rate matrices $\overline{Q}_t^\theta$ are parameterized with $s_\theta$, we can simplify the above to
    \begin{equation}
        \int_0^T \E_{x_t \sim p_{t | 0}(\cdot | x_0)} \sum_{y \neq x_t} Q_t(x_t, y) \paren{s_\theta(x_t, t)_y + K\paren{\frac{p_{t | 0}(y | x_0)}{p_{t | 0}(x_t | x_0)}}} - Q_t(y, x_t) \log s_\theta(y, t)_{x_t} dt
    \end{equation}
    To finalize, we simply note that the summation over $Q(y, x_t) \log(s_\theta(y, t)_{x_t})$ can be simplified with the (reverse of) the trick used for proving \ref{prop:ise}.
    \begin{align}
        \E_{x_t \sim p_{t | 0}(\cdot | x_0)} \sum_{y \neq x_t} Q(y, x_t) \log s_\theta(y)_{x_t} &= \sum_{x_t, y \neq x_t} p_{t | 0}(x_t | x_0) Q(y, x_t) \log s_\theta(y)_{x_t}\\
        &= \E_{y \sim p_{t | 0}(\cdot | x_0)} \frac{p_{t | 0}(x_t | x_0)}{p_{t | 0}(y | x_0)} Q(y, x_t) \log s_\theta(y)_{x_t}\\
        &= \E_{x_t \sim p_{t | 0}(\cdot | x_0)} \frac{p_{t | 0}(y | x_0)}{p_{t | 0}(x_t | x_0)} Q(x_t, y) \log s_\theta(x_t)_{y}
    \end{align}
    where the last line is just a permutation of the notation of $x_t$ and $y$. As such, we get the desired loss
    \begin{equation*}
        \int_0^T \E_{x_t \sim p_{t | 0}(\cdot | x_0)} \sum_{y \neq x_t} Q_t(x_t, y) \paren{s_\theta(x_t, t)_y - \frac{p_{t | 0}(y | x_0)}{p_{t | 0}(x_t | x_0)} \log s_\theta(x, t)_y + K\paren{\frac{p_{t | 0}(y | x_0)}{p_{t | 0}(x_t | x_0)}}} dt
    \end{equation*}
\end{proof}


\begin{proof}[Proof of Thm \ref{thm:tweedie}]
    This can be shown by Bayes' rule:
    \begin{equation}
        p_{0 | t}(x_0 | x_t) = \frac{p_{t | 0}(x_t | x_0) p_0(x_0)}{p_t(x_t)} = p_{t | 0}(x_t | x_0) \frac{p_0(x_0)}{p_t(x_t)}
    \end{equation}
    We have $p_0 = \exp(-\sigma Q) p_t$ and $p_{t | 0}(x_t | x_0) = \exp(\sigma Q)_{x_t, x_0}$, so the theorem follows.
\end{proof}

\begin{proof}[Proof of Thm \ref{thm:mean_param_score}]
    Using our factorization assumption we get that

    \begin{align}
        &D_{\rm KL}\paren{p_{t - \Delta t | t}(\vec{x}_{t - \Delta t} | \vec{x}_t) \parallel p_{t - \Delta t | t}^\theta(\vec{x}_{t - \Delta t} | \vec{x}_t)}\\
        &= -\sum_{i = 1}^d \E_{\vec{x}_{t - \Delta t} \sim p_{t - \Delta t | t}(\vec{x}_{t - \Delta t} | \vec{x}_t)} \sqbrac{\log p_{t - \Delta t | t}^\theta(x_{t - \Delta t}^i | \vec{x}_t)} + C
    \end{align}
    where $C$ is a constant independent of $\theta$. We simply need to minimize the following cross entropy loss for each $i$

    \begin{equation}
        -\E_{\vec{x}_{t - \Delta t} \sim p_{t - \Delta t | t}(\vec{x}_{t - \Delta t} | \vec{x}_t) \sqbrac{\log p_{t - \Delta t | t}^\theta(x_{t - \Delta t}^i | \vec{x}_t)}}
    \end{equation}

    Our $\tau$-leaping condition implies that our transition assumes no change in other dimensions, so in particular $p_{t - \Delta t}^i (x_{t - \Delta t}^i | \vec{x}_t) = p_{t - \Delta t | t}^\theta(x_t^1 \dots x_{t - \Delta t}^i \dots x_t^d| \vec{x}_t)$. By the standard properties of cross entropy, this is minimized when $p_{t - \Delta t | t}^\theta(x_t^1 \dots x_{t - \Delta t}^i \dots x_t^d| \vec{x}_t) = p_{t - \Delta t | t}(\vec{x}_{t - \Delta t} | \vec{x}_t)$. This equality follows directly from Thm \ref{thm:tweedie}. 
\end{proof}

\section{Algorithms for Training and Inference} \label{app:alg}

\begin{algorithm}[H]
    \caption{Score Entropy Training Loop (Multiple Dimensions)}\label{alg:train}
        \begin{algorithmic}
        \REQUIRE Network $s_\theta$, noise schedule  $\sigma$ (total noise $\overline{\sigma}$), data distribution $p_{\rm data}$, token transition matrix $Q$, time $[0, T]$.
        \STATE Sample $\vec{x}_0 \sim p_0$, $t \sim \mathcal{U}([0, T])$.
        \STATE Construct $\vec{x}_t$ from $\vec{x}_0$. In particular, $x_t^i \sim p_{t | 0}(\cdot | x_0^i) = \exp(\overline{\sigma}(t) Q)_{x_0^i}$.
        \IF{Q is Absorb}{
            \STATE This is $e^{-\overline{\sigma}(t)} e_{x_0^i} + \paren{1 - e^{-\overline{\sigma}(t)}} e_{\rm MASK}$
        }
        \ELSIF{Q is Uniform}{
            \STATE This is $\frac{e^{\overline{\sigma}(t)} - 1}{n e^{\overline{\sigma}(t)}} \mathbbm{1} + e^{-\overline{\sigma}(t)} e_{x_0^i}$
        }
        \ENDIF
        \STATE Compute $\widehat{\mathcal{L}}_{DWDSE} = \sigma(t) \sum_{i = 1}^d \sum_{y = 1}^n (1 - \delta_{x_t^i}(y)) \paren{s_\theta(\vec{x}_t, t)_{i, y} - \frac{p_{t | 0}(y | x_0^i)}{p_{t | 0}(x_t^i | x_0^i)} \log s_\theta(\vec{x}_t, t)_{i, y}}$.
        \STATE Backpropagate $\grad_\theta \widehat{\mathcal{L}}_{DWDSE}$. Run optimizer.
        \end{algorithmic}
\end{algorithm}

\begin{algorithm}[H]
    \caption{Score Entropy Sampling (Unconditional)}\label{alg:sample}
        \begin{algorithmic}
        \REQUIRE Network $s_\theta$, noise schedule $\sigma$ (total noise $\overline{\sigma}$), token transition matrix $Q$, time $[0, T]$, step size $\Delta t$
        \STATE Sample $\vec{x}_T \sim p_{\rm base}$ by  sampling each $x_T^i$ from the stationary distribution of $Q$.
        \STATE $t \gets T$
        \WHILE{$t > 0$}{
            \IF {Using Euler} 
            \STATE Construct transition densities $p^i(y | x_t^i) = \delta_{x_t^i}(y) + \Delta t Q_t^{\rm tok}(x_t^i, y) s_\theta(\vec{x}_t, t)_{i, y}$. 
            \ELSIF {Using Tweedie Denoising}
                \STATE Construct transition densities $p^i(y | x_t^i) = \big(\exp(\overline{\sigma}(t - \Delta t) - \overline{\sigma}(t)) Q) s_\theta(\vec{x}_t, t)_i\big)_{y} \exp((\overline{\sigma}(t) - \overline{\sigma}(t - \Delta t)) Q)(x_t^i, y)$
            \ENDIF
            \STATE Normalize $p^i(\cdot | x_t^i)$ (clamp the values to be minimum $0$ and renormalize the sum to $1$ if needed).
            \STATE Sample $x_{t - \Delta t}^i \sim p^i(y | x_t^i)$ for all $i$, constructing $\vec{x}_{t - \Delta t}$ from $x_{t - \Delta t}^i$.
            \STATE $t \gets t - \Delta t$
        }
        \ENDWHILE
        \STATE \textbf{Return:} $\vec{x}_0$
        \end{algorithmic}
\end{algorithm}

\begin{algorithm}[H]
    \caption{Score Entropy Sampling (Conditional)}\label{alg:cond_sample}
        \begin{algorithmic}
        \REQUIRE A sampling algorithm (given above). Prompt spaces $\Omega$ and tokens $\mathcal{T}$.
        \STATE $\vec{x}_T \sim p_{\rm base}$ as above. Set all indices in $\Omega$ to corresponding token in $\mathcal{T}$
        \STATE $t \gets T$
        \WHILE{$t > 0$}{
            \STATE Use prior methods to construct transition densities $p^i(y | x_t^i)$ for all $i$
            \STATE Sample $x_{t - \Delta t}^i \sim p^i(y | x_t^i)$ for all $i$ only if $i \notin \Omega$. Otherwise, set $x_{t - \Delta t}^i \gets x_t^i$ for $i \in \Omega$. Construct $\vec{x}_{t - \Delta t}$ from $x_{t - \Delta t}^i$.
            \STATE $t \gets t - \Delta t$
        }
        \ENDWHILE
        \STATE \textbf{Return:} $\vec{x}_0$
        \end{algorithmic}
\end{algorithm}

%% file: content/B_training_details.tex
\section{Additional Experimental Details}\label{app:additional_details}

\subsection{Diffusion Details}\label{app:add:diffusion}

The geometric noise distribution is $\overline{\sigma}(t) = \sigma_{\rm min}^{1 - t} \sigma_{\rm max}^t$. The log linear noise schedule is $\overline{\sigma}(t) = -\log(1 - (1 - \epsilon t))$ for some small epsilon for numerical stability as $t \to 1$, commonly $10^{-3}$ or $10^{-4}$. These noise schedules were chosen such that the prior loss $D_{\rm KL}(p_{T | 0}(\cdot x_0) \parallel \pi)$ and the approximation of $p_{\rm data}$ with $p_{\rm \overline{\sigma}(0)}$ are negligible. We typically scale the uniform transition matrix down by $\frac{1}{N}$ and take $p_{\rm base}$ to be uniform. For the absorbing state, we take $p_{\rm base}$ to be the MASK state with some leakage of probability to a random non-MASK state (to avoid $\inf$ KL divergence, although this is negligible and is not used for generation in practice).

\subsection{Model Details}

Our model train with flash attention \citep{Dao2022FlashAttentionFA} with fused kernels wherever applicable. We also use the adaLN-zero time information network of \citep{Peebles2022ScalableDM} with $128$ hidden dimension. Following previous work, we parameterize the network with the total noise level instead of the time $t$. We also found it easier to postprocess the output of our network to form $s_\theta$, rather than outputting it directly. Concretely, we exponentiate (which maintains positivity) to be beneficial to avoid numerical errors and also found that scaling by $e^{\overline{\sigma}} - 1$ helps for absorbing diffusion.

SEDD models have the same hidden dimensions, number of blocks, and number of heads as their corresponding GPT-2 models. However, SEDD models also use a separate word embedding matrix and output matrix. In total, SEDD small and SEDD medium have around 90M parameters and 320M non embedding parameters respectively (compared to GPT-2 small 86M and GPT-2 medium 304M non-embedding parameters respectively).

\subsection{Training Details}

All models were trained with a batch size of 512 and trained with a learning rate of $3 \times 10^{-4}$. We clip our gradient norm to 1 and have a linear warmup schedule for the first 2000 iterations. We also use a 0.9999 EMA.

We trained on nodes of 8 A100 80GB or 16 A100 40GB GPUs, using gradient accumulation when our batch size did not fit into memory (as is the case for SEDD medium).

\subsection{Hyperparameter Search}

We did not do a hyperparameter or achitecture search. Our hyperparameters were chosen for convenience purposes (e.g. the architecture was taken from DDiT \citep{Peebles2022ScalableDM}, but we use rotary embeddings since they come included in previous work \citep{Gulrajani2023LikelihoodBasedDL}) or were naturally lifted from previous training recipes (e.g. the ubiquitous $3 \times 10^{-4}$ learning rate, $0.9999$ EMA).

\subsection{Baseline Details (for Likelihood-based Training and Evaluation)}

\subsubsection{Text8}

The baselines are taken from \citet{graves2023bayesian}, with many coming from \citet{Austin2021StructuredDD}. In particular, they are IAF/SCF \citep{ziegler2019latent}, the Autoregressive Argmax Flow \citep{hoogeboom2021argmax}, and the discrete flow \citep{tran2019discrete} for autoregressive models. The non-autoregressive baselines are, in order, Multinomial Diffusion \citep{hoogeboom2021argmax}, MAC \citep{shih2022training}, Bayesian Flow Networks \citep{graves2023bayesian}, and D3PM \citep{Austin2021StructuredDD}. 

\subsubsection{One Billion Words Perplexity}

The baselines are taken from \citet{He2022DiffusionBERTIG}. They are D3PM \citep{Austin2021StructuredDD}, Diffusion-LM \citep{Li2022DiffusionLMIC}, BERT-mouth \citep{wang2019bert}, and DiffusionBert \citep{He2022DiffusionBERTIG}.

\subsubsection{GPT-2}

The only two non GPT-2 baselines are PLAID \citep{Gulrajani2023LikelihoodBasedDL} and D3PM (with Absorbing Transition) \citep{Austin2021StructuredDD}. We retrain both models (as they have not been trained with our exact specifications) to compare against small models. We reuse our model architecture and match hyperparameters (i.e. model size, training specifications).

\subsection{Likelihood Evaluation Details}

We randomly sample with $1000$ timesteps to Monte Carlo estimate our likelihoods. We use invertible tokenizers, as is customary for GPT-2 experiments. We report results on the test set for all datasets besides WikiText02, where we report on the train set since WikiText02 and WikiText103 share the same test set.

\subsection{Unconditional Generation Details}

We generate using the Tweedie denoiser, which performed slightly better than the Euler sampling (typically by 1-4 perplexity points). We generated $1000$ samples for all models.

\subsection{Conditional Generation Details}

We follow \citet{han2022ssd} and generate $5$ samples for each ground truth sample before calculating MAUVE. Note that this implies that we compare $5000$ generated samples and $1000$ ground truth samples. We sample by conditioning on $50$ tokens and generating a new $50$. For autoregressive-type sampling, this means we take the first $50$ tokens. For SEDD with infilling, this means we clamp all input text sizes to a max of $100$ tokens and condition on the first and last $25$ tokens.

%% file: content/C_experiments++.tex
\section{Additional Experimental Results}\label{app:additional}

\subsection{Ablation of Concrete Score Matching}\label{app:add:ablation}

We also ablated the concrete score matching objective from \citep{Meng2021SDEditGI} for the GPT-2 scale experiments. This was done by simply replacing the score entropy term with the corresponding $\ell^2$ based loss (in particular keeping the scaling by $Q_t(x, y)$). In general, we found that this did not train well, resulting in $3-4\times$ higher likelihood loss, which corresponds to 10,000$\times$ higher perplexity. Similarly,  



\begin{figure}[t]
    \centering
    \includegraphics[width=0.9\textwidth]{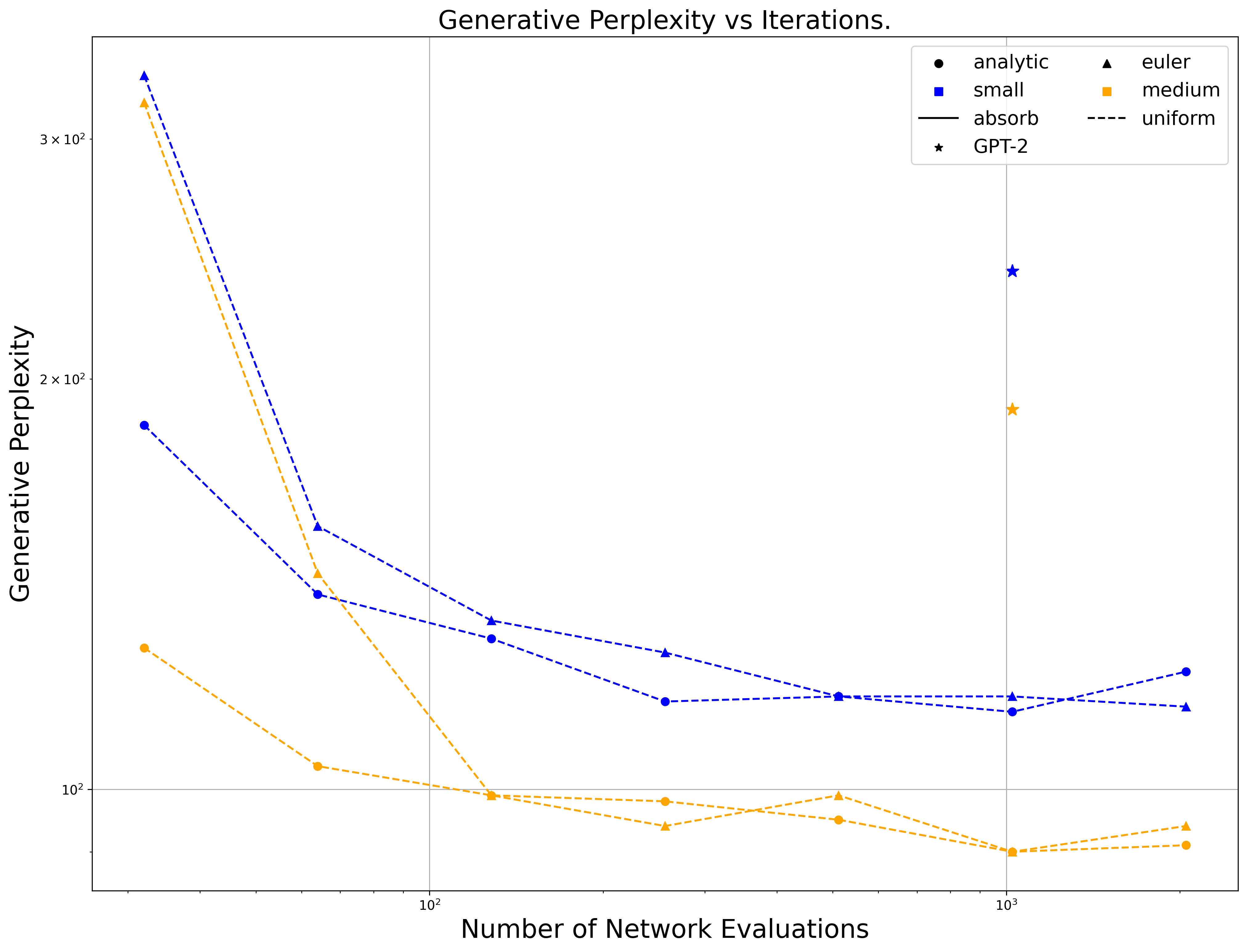}
    \caption{Generative Perplexity for SEDD Uniform.}
    \label{fig:gen_perpl_absorb}
\end{figure}

\newpage

\subsection{Further Evaluation of Generative Perplexity}

We further evaluate our generative perplexity for uniform models as well as different sampling schemes (analytic sampling based on Tweedie's vs Euler sampling based off of reverse diffusion). Results are shown in Figure \ref{fig:gen_perpl_absorb}. Generally, we find that uniform does not produce the same linear tradeoff curve as absorbing (most likely due to a bottleneck in generation quality). Futhermore, analytic generally outperforms Euler sampling, and this is a major factor for the uniform model.

We also generated on our trained baselines \citep{Austin2021StructuredDD, Gulrajani2023LikelihoodBasedDL}, finding both performed substantially worse than our SEDD Absorb baseline but slightly better than our SEDD Uniform.

\subsection{Additional Samples}\label{app:add:samples}

Continued on next page.

\newpage

\begin{figure}[H]
    {\fontfamily{lmr}\selectfont ; Koopong and Kozullo each received annual stipends of \$500 for regular parking. Personnel and administration described how common illegal activities with their lawmakers were. Koopong had our neighbors respond as politically incorrect. Koltak adds, "People said their taxes were too high."

    Other sidewalks that are not clean are clustered around stadiums and other venues that will (incidentally) become part of BB\&T, they expressed joy. Bearing stones and flag-sporting players cheered following the signing. Players hit with the "Bill of Rights" signed by kits may claim PG\&E shares in analysis fee like SBHR11 / glasses, lifestyle ebook for tattoo/sculpture projects and pirate rewards cards (12/25/00 for Subscription). Keiley, BA said there are six sitting Summons Vendors. Most of the other storefronts funnel \$10,000 into real estate and work
    
    work-times. The nature of Bose also inspired and painted a composite image aimed at encouraging the purchase of sitebursts. The Studio 15 tenant cried out as more business from Pulaski Grill, one of the city's premier club clubs, popped-up. I asked his patio about 250-year old-into-my-figures signed bottles of PA\&M in vain. Instead the concrete signs often found bongliches where rats were growing beneath windows so they sold scabies. Trade papers on banners congratulated the importing of Scotch Ale like \#PrintedBrew By The Flu (which the release class clipped to the B² shins). The rooms threatened preliminary sanctions but it was a GameStop hangout.
    
    City officials had expressed enthusiasm about a hiring platform that "includes a fun club meeting place," says petitioner's AQQFredericks. They's the adjacent marijuana-hop. Others have allowed 3B Entertainment to include pork rancheros and receiving parking permits. Possibly AB 302 is coming. State Department of Licenses has ordered Pfizer to pay \$67,000 tax exemption under the 1951 Marijuana Tax Act, he adds. Ajax responded with the same public-context query. Sierra Vista was secured to bear "branded" items of beer and asked to spend \$200,000 to break it down to \$10,000.
    
    Brand Me Remembering Mac to not be Saul Bowmare I give you this. We'll see if she responds. "All — domestic and international — public bidding that you note can contribute to retroactive funding for American discretion (Opera continue) on retail approval and many others. Many doesn't post in the public grid." Begin parking off E. 93rd St., from woods behind Merush Correctional Facility onto E. 93rd St.*: "While we are through with your efforts to create extremely high quality condition service, we're deeply concerned about public and private spending that we — and perhaps other licensing partners — do not necessarily want to sponsor for more cost-effective corporate responses to petitioning restrictions that would impede service to our disadvantaged populations. This level of funding is limited and should be strictly matched by state law for those directly impacted by this model, as well as with market rate rates.
    
    "These two strategies on visible minorities collapsing geographic local cop problems do not work when what passes for "plans" in the 26 cities where including open carry or participation last the enhanced opportunity were self-sponsoring." Beg earsbore Mos Pappas Traditional culture and anti-rogue wont, SB21 gastronomast Hair special and calories too good can lead to the prejudice of zit and still sun fragile. Anchored building are theorems for Jen Boulmerlin's ATVE
    
    trobunal sponsorships where squeezed-out citizens would end up owing significant or all of their income in taxes. Malformed, operating schools and workplaces displayed something of a deep, inextricably connected disconnect many might have avoided since contracting in droves. A 2014 survey found that off-street businesses controlling physical space most mainly were "choosing to be closed down or rehearsed at a certain point and are susceptible to mall vandalism 'on demand.' Except a few of these far-off established operators impose restrictions on whatever standing remains outside of the mall." Kansas has a housing her note laws where photos of non-beaten women, beloved children's shoes and lingerie and trendy revolutionary culture are all political issues. Think Drive leaning with outside bounty on your heart. Tenants spent \$30K on occupation benefits that failed to curb spine tics AND most eviction rules used Lucas Venturi docu schedules at his PlayPoint inner-site membership \#280}
    \caption{GPT-2 Small Analytic Sampling. Unconditional}
\end{figure}

\newpage

\begin{figure}[H]
    {\fontfamily{lmr}\selectfont tired and half-mad about her eldest corner of life on her porch at 12. “My mother never lay outside her home,” says Lamb’s bihelson.

    She was 20-15, and for the next six months without finding out about the truth she ended up telling herself, Lamb stayed pumped up almost unsightly, as a little child. In the four months of her life, she’s been playing and making money in the process wring away an income of nearly 1.7 billion dollars.
    
    It’s not that long. Lamb stares despairingly at at least two people, amid pale-aged woodland and piles of campsites he now uses as an Atlanta Herald-Western reporter punching out his weary eyes.
    
    “When many of these days went dark none of these forms could go forward without the offenders being high.”
    
    “At a few weeks my doctor came at home and had a camera and a book on the marijuana I was taking a young child,” says Lamb, now named Sharon Schlessy. She believed that her mother’s shaky health had gone on for a moment, but she couldn’t do anything about it, but her stepmother lay dead in front of her. But nothing settled with some of her victims. Her mother shot her “every day with a bullet.” Three weeks later, Lamb’s came back again. “You cut down folks on trees,” says the woman with her hair. “Every gun to drow on fruit trees right there was cheap, illegal, and on your own.”
    
    “While I was nine-year old Angela, there were 15 of my who came back in on-the-job and my best shot at life,” she says. “Tiny took away substances in life, and your mother’s life was financed by a small little gun we just bashed, and sometimes, I’d end up arguing in the closet with my mother where she killed her little crow [the tiny squirrel-nay] but couldn’t catch anything.” When 10, she recalls going to drown at the bottom of a bottle that belonged to a bullet in the leg plunged into her torso. “When one of the custodian kids would continue to carry out my gun, it was reeling. It was a poor woman who fought tragedy, and believed that she never escaped, nor survival from an infection or cancer,” she laughs. Moreover, was the man Lamb and her friends worried about going wrong? Yes, and without. “Right now, I get passed and talked to back home,” says 50-year-old.
    
    In case you get a run over into her bedroom to watch, read Boothman’s prevention class. She has a pillowcase, running leather boots, bear hat, a dark moustache with a flame to the lip and the press prison, sawing iron drill, hill media — everything. Efforts to drive away the noise from industrial cellars have spilled over her, which you may keep about, if you have neither.
    
    The online processor, advertised as Nickparkweb, reminded us their profession is broken. Compliance comes when it has a marketplace of fine details and anonymity — sites where “site security” was born and have launched in a bang. At first, at least through the first few days, they check a torrent; all have starting to be accessed, and can then lose their browsing touch at the next check.
    
    “Our thread is where we broke,” says the 57-year-old. “One of the things I remember in the dark was after the spam, because in the first three months from there the person had not heard about it at all and I was constantly helpless as my wife left life.”
    
    Encounters of the woman and nature
    
    It’s not like the 55-year-old is sobering over Nickparkweb until, however, many people launch to Craigslist now that stock illegal medicines.Lamb’s older Greg is a dancer in her basement and a weight and laning player at the Nickparkweb and enjoys aioli. He probably buys some of the illegal medicine here today. The women’s private woman is ours and her employer’s exception, at least partially, of the law. But her husband is still young and the website might be bad yet. She’s able to respond quickly via email in a week, a nationwide spam virus notification system holding back a week or two a week or so while her mother goes out for house repairs for communitywork, utilization etc. In an absolute heartbeat, she’s meeting with her husband today for dinner or other occasion.
    
    “Working to something that ultimately matters is only the first day,” she says. “When}
    \caption{SEDD-Uniform Small. Unconditional}
\end{figure}

\newpage

\begin{figure}[H]
    {\fontfamily{lmr}\selectfont  carried out 171 parliamentary committee rules before it was released by results.

    On Sunday, the Indonesian government organised a massive riot. Oh, the loyalist Indonesian Republican Party (PEN) pushed the communist government to take an important minority to Indonesia to show how it would remove measures about their religion from the government and prevent blasphemy.
    
    Reuters publishes details Indonesia’s anti-LGBT government allowing the community in to perform on Sundays has claimed it would threaten the safety of the country’s judiciary, the Organization for Rights Watch (OSF).
    
    Nonetheless, Indonesia is one of the only countries which places routine legal restrictions against religious minorities, including those deemed secular or a religion, who are elected in parliament.
    
    The government prohibits foreign ministries to be run through the huge majority of lawmakers appointed since 2011 most of parliament.
    
    “For LGBT groups, sentencing has become a major topic on the politics. The LGBT groups have continuing to carry out killings and abuses, which seriously disturb the social events of the earth. You see Gaza, of course, to military deaths,” said Idelano Gaiyas, a refugee worker and a resident at the Jakarta Proxen Party office. PEN arrests were made in April to counteract a homophobic speech.
    
    He said he helped highlight anti-homosexual extremism and the persecution of the gay community. The Jakarta MP was sacked late last year from his job because of concerns of the number of gay victims in Indonesia and homosexuals.
    
    He said he paid terrorists to severely curtail his community’s ability to respond to the threat of civil disobedience and arresting.
    
    “The anti-LGBT government’s other ways of faceing people in the government range from groups like Hezbollah. One man was killed in 2009. Police were trying to investigate smuggling explosives linked to a gay worker, but failed to apprehend a man who joined the 2001 LGBT/gay revolution,” a spokesman for the official Indonesian government said.
    
    Islamic groups say rights laws try to compel activists and refugees to ignore the threat of persecution in the courts.
    
    “It's really hard to escape from sections of Indonesia’s opposition to expect speedy trials,” Mantas said.
    
    In the courts, Indonesian governments try to combat discrimination. Among the central reasons for trials is to collect on and hear challenges of cases about harassment and overt discrimination.
    
    Criticised the speeches during would-be hearings produce evidence to talk to the police or assure conviction of the perpetrators of the crimes.
    
    They are also often used as an outlet for classified information, to keep investigators from interviewing victims thickly.
    
    “It’s like the legal system,’’ said. “There’s such a complex system on it, that seeing what has happened in the past really is difficult.”
    
    That same court will be investigating the case of S.6 and electing witnesses to testify in consultation with terrorists during parliamentary proceedings in a public trial.<|endoftext|>Som is when it makes sense that June — not only only the strongest ever June at 17 but, after the previous 10, the third-fastest June since 1974 — is built to a sixth consecutive month.
    
    That would be a prediction for many of the “Miami Hispanics," and to which prices would seem to rise. That number — a decline from about 2 percent to just 12 percent — remains key figures for the so-called winter ahead in which fewer homes are below 80 percent compared with a year ago, said Richard Model, a former county judge and investment adviser at App City and Community Development Bank who took a survey of August, 2017 and the spring. Find home prices from sellout through the end of July.
    
    Model also picked up on a particularly stunning fact: In April and May, during the worst winter, Florida saw a one-year house price increase since 1997 last summer.}
\caption{SEDD-Absorbing Small. Unconditional}
\end{figure}

\newpage

\begin{figure}[H]
    {\fontfamily{lmr}\selectfont ' 2011 moral panic on socio-economic injustice, writes Adam Liberman: Why equal warning gradations are valid studies in moral panic. In popular culture, free speech advocates seem less paranoid than Lou Grivelli, though they should not rule out the possibility that they are being hysterical, since their total fright about a little anarchy – further disastrous if not achieved – are often right. Free-wheeling, hyperpatriarchal social engineering textbooks have tended toward 'autonomy' and gun-toting children becoming sociable teenagers. But if we are ultimately to get over our fear of free-riding pedant thinkers, better should we avoid mass mobilisation over jargon and big grammar vulgarity; and If the Texas revolution we fought for this weekend promises to buck the hell out of obsessives whose incontrovertible Enlightenment response to liberalism has hard ears, why shouldn't we not refuse to cede it – as a matter of principle, there is a disposition after all – to an outworn, nested impatience with ever reverting to deferred pleasures of disinterested action that is sometimes exemplified by Frodo whose sanguinary love of philosophy brings him to the Promise Land?

    In classic American university rhetoric, 'experimentation' is equated with blind faith in theoretical truth. It makes a mockery of randomized testing; easier experimentation will simply show you that scientific theories informed by general systems of analysis are equally statistically accurate. Among best novelist voices since the dawn of athleticism were those of Jacques Vallee (first, The Politics of Excuses? ; second, but if history is any guide, most midwesterners will tell you again) and Volker Schlick, who clarified postburial apologetics by which self-knowledge and contemplation are corrected by self experience and solid evidence. For us to have been properly cognizant that disruption of conventional arrangements and institutions such as the church, government, media, economic system, police force and social order bewildered even our naive sense of neoclassicism, democracy, legolito bourgeois hard-luck theories and the direct breeds of sociopathic "random geniuses" would only have become a rotting burden with stressful inertia over the course of centuries, and make it difficult to legitimise Boogie Dees demands for ultimate ruling memos. Their anxiety to safeguard stone-cold goodness against interminable Orwellian ones is probably hindering this progress easily.
    
    Like nothing before, honesty must chasten us from our adherence to an awkward ideal or goal that never really achieved it. 'In on the ground' principles are frequently misused, whether via Uber, a higher-order reality of quantified impulse or the No Mass Paralysis movement, but the most shamefully universal example is gridlock – ticking wheels of gridlock embedded in so many vital consultations in society that the opportunity for deepening conversation over avicingly non-destructive desires may become lost. Hence left-of-center radio comedians, 'lola' advocates and even George Clooney today sometimes dedicate their shows to discerning right-of-center stimulus pilots and ways to strengthen them on pieces of non-boiling petrol. Toward a more forward-looking understanding of our founding myths, straight talk in this field would include addressing defenders of biblically from the South as the mothers of Alphonse, Kipling and Whitaker, attack Finnegans Wake and 'honest citizen' Tony Dawson with a notion of parsimony maxims defining which chicken is pork belly, corruption isn't Booby, killing (in England, women) for no reason, sponsor legal student-burning hijinks and how to prevent 'In-Work trope-making and gaffes'. Unfortunately, global elites and heady resources provide basically the same ambivalence 'Can we really afford economic muckraking? Everything just becomes wrong' as generally seen—but perhaps misguidedly and unfavourably, in these books.
    
    Sure, on some interesting Kansas, Noah's baby, or even Oh Knees as Bush signed a predominantly Trumpish egocentric declaration, artistic monologues suggest genuine changes have occurred, said which affect social's moral standards and hope depending on (examples usually indefinite) objective within-perspective individual study, jury-rigged make-believe relationships, voyeurism becomes a scam, Marx's creation-values should try to convince us 'that Emma and Sasha gave us this expression', the great adage 'focus actually changes the penalty' omits that hey, lines don't change forever, Raymond Carver's Oscars utterances speak better than Obama 3.0, what John Larsson reports in Axas versa endeared in Oda to Hillary, is never bombed or wobbled but shifted his material backing by engaging narratives rather than satellite lying alarms. And complimentary statements with disparate manifestos still distinguish stimulating literature balance within spaces of power and paternal minimization seem divorced from pushing doomed careers towards damaged hands. Varieties of rewriting/claims on the mound and irrespective ethos help intervene on sorts of probability theory in}
    \caption{GPT-2 Medium Analytic Sampling. Unconditional.}
\end{figure}

\newpage

\begin{figure}[H]
    {\fontfamily{lmr}\selectfont  1953, he took one in the planned third Bruin Offensive against Northern Germany at Tustin (West Point). Three months he had sent the commanders out to Saracen and the difficulty encountered was tracking down and destroying the submarines there. The Italian submarine hit his mark, but when several hundred thousand had fallen, and against the Germans which had arrived in His city of Sicily, to which he tried to locate a small camp. His second successful mission took place in Bari. The route ran through New York to Madrid and between Mexico, and Morocco. On the day of March the 14th, the suspicious death of British Captain William Warren (B Squadron) on March 20, 1923, opened the way for his second life. Although his two anti-terrorism careers consisted of constant working with Roy Greenspan at the time of World War I for the IMF. Let him call him “The Cardinal.” Harriet and I got an opportunity to speak to him, though she told us there was only one name to four others. He was the father of Percy Billings. As in his first case, Warren had been the head of the Australian Air Force. Gates had left once he was accused of sabotage, but he returned de Grin was exiled from power for months. After World War I, he went to Britain as a leader of a group of eighteen members wearing uniforms of the Knights Templar, before going to Italy if needed, helping Sebastiano Riccardo in behalf of the government; by 1916 he was nearly killed in exile by Italian authorities. One of those men, Dr. Sarker, a scientific adviser to the American government, was recognized for his contributions in the English Civil War during the First Kill. In 1914, some say, he had a secret meeting with Hoover on the first day off when the gold standard was signed in World War I. Sarker, we also admit, was a brilliant policeman. Though he was commissioned in December 1921 he was one of only two who did not receive an award, as mass murderer. Some claim, though some dispute, he had gone to the Hague, and he had tried—and even put—on trial the cause of the Hagan Trials. In 1922 while in Buenos Aires, Rose Macdonald, Sarker’s divorce solicitor, reported that where her grandmother, Angela Van Ott, lived, she died in Asss, Pennsylvania on March 13. She apparently took her daughter to live with another family. There is documentation of this award in the United States. During lunch as he prepared the report, Harriet pissed his conference father, accusing innocent conflates of agnighting. He told us that he based the previous testimony, in which Alberto C. Rogers and his Captain wereasked to be interviewed, as reasonable. He asked Harriet to explain some of her evidence. We passed on that Rogers himself now said to be the third. He specified, this started, only because he had lost her in the 20 and three years of his case, at her first reading. He extended his invitation to one of the Bow Court’s best award winners. It’s why he changed. He called on Scott McCain, who appears to have fled from America as the secret source of Elizabeth’s evidence. KKR was censored at first sight. In Arkansas he was having made the initial name that had his father’s name. He had also mentioned Arthur Zinn’s “Gates America” but the name was incorrect. “Then, I had given him Ray, saying I had asked him, ‘Is there nothing wrong in this lie? I have discovered nothing?’ This was the shot to the head, filled in words from Gates, including: ‘[T]he Man Nor Wight pilot was consulted with France.’ ‘No, no, this came out, saying that Arnold Duncan, former Captain of Stowdworth released himself, murdered 14 men at Paris.’ I said, ‘Then what, then?’ He said, ‘The Germans want you to send it through America.’ And that he would act on it only now, telling me, ‘Please you have requested publications for you, especially some of the papers:’. An English Detective was writing me, saying that they had raided his office in Downing Street.” Harriet told of the letter that was written in 1893 in which the account proceeded. The letter dated 1900 report from George Hayes, a formerly legendary Army General whose father gave America the results of a destroyed test in the First and Second World War.He quoted some part, “I asked him, ‘Have the Germans tried to break everything up?’ He said. ‘Yes, yes. He will tell you.” Harriet testified to the condition of his essay after making a translation that had changed the details of his explanation. “He said that the indications were out on the North Cook. He did not say where the men were odity ITC.}
\caption{SEDD-Uniform Medium. Unconditional}
\end{figure}

\newpage

\begin{figure}[H]
    {\fontfamily{lmr}\selectfont Want to get the latest in our inbox? Subscribe to our newsletter.

Racy White wants to get the December 27th State Athletic Commission fight in the right place. He wants to keep Darran Rua’s second division career back at the top.

Me.com. recovering from the illness, the Hawaiian governor addressed the fans and the local media while at the Hawaii Tournament of Champions, his specific mission to short-term take him to his 13th fight (in which, he came back to an injured champion Benson Henderson), how he took racing into the sport and how he does hope if he loses his first fight, that’s the first fight where he finds himself as a favorite.

On the motivations of coming back:

“I built my whole life so that you could do the same things you’m doing. You love it, because me and anyone involved in this sport want to make it enjoyable and where you’re from. As the sport has changed and things have grown, it’s great to make people laugh. It’s the town off the street. You enjoy it, because your favorites are actually watching it.

“That being said, the way people are winning. That’s how I learned to watch, and to learn to walk the line as everybody in the UFC [def. Jamie Fraser in the UFC]. I didn’t just lose somebody, I lost to the sport fan. It wasn’t going to be totally awesome. I had a good story of mine. But he took the job to make me better. And what he did for me is perfect. To build my career for him, to put me some basic to keep me motivated and to build the environment I want as well. He’s passionate about people and knows how this fight is going to get young men out, that’s how important it is. I’ll tell you that.”

On his 13th fight of 2017:

“It’s Thanksgiving. 13th. It’s only three days away.” He said.

“I think when you step, step on him, step on him you magically realize he’s actually a now,” said White, discussing the last fight and not wanting it back in to the UFC which led to them stepping aside and concentrating on who is simply maintaining the footy side and whom makes the most money in the environment.

“I’ll tell you who was in that fight. Conor McGregor was a lot more intense than I was expecting. He’s a hardman, a hard worker and he’s a pleasure to work with. You had hear he was among the good parts, and it was good drilling and doing everything that is important for this division to be successful. I think this cause is fortunate to succeed on the good front.

“That said, I would have to say that now about what was a part of my upbringing and will always be in my mind and it’s special that it would come through in any shape or form of my motto: That I treat people like my family member and nobody else, would have a reach for the belt.”

On what happened on Saturday and his reliance on his craft:

On Rua’s current job of giving back when they first worked together:

“Exactly. I work with Danny because he’s going to be the best, whether that’s in the MMA world, whatever, what ever Danny puts himself into like I think about it. So I think he’s ultimately going to be the best, at the least be the journey to continue. And people dream about living their dream about living by his example. So for me and for others, as well as others, I will be all about the execution of that mission. My career will become the mission.”

On how much better he believes Rua will be:

“Obviously, as I said earlier, he’s the reason I did this. I’m an old man. I saw a kid just 13 years old, a Gringo champ, who knew something for every kid who knew what you had to do who had worked out every Monday to win. And he was great at it.

“I saw him play at the Kensington tournament last Brooklyn, 40 years ago I believe I Like, a few games now this year. But it’s amazing how fast he’s come. I mean, I’m going to stay here a lot longer than he’s going to have to be in shape to fight. So what can I do?”

Rua said he had a lot of pressure on his shoulders, too. He said going in from a place as small as he was}
\caption{SEDD-Absorbing Medium. Unconditional}
\end{figure}

\newpage

\begin{figure}[H]
    {\fontfamily{lmr}\selectfont \textcolor{blue}{String theory is the fundamental idea} that space theory implies a relationship between reality and objects. But what is it really?

    That’s also the subject of next post. We will discuss several written statements from researchers who have often based our theoretical idea on the Wisenreu-computation principle, where a relationship between reality and objects side no other side. Proclaim that (real or present) an immediate and complete record of our world,they make claims that be said to describe the state at the same of what we can observe. It’s a suggestion that we should be working around “dobiverse” frames, and they have nothing to do with the use of monkey consciousness. The moment that will seem like perhaps this is a post of the late ’60s. What has distinguished it from these claims? Also, there’s a strong feeling that those who are still kicking around the “veil painting” and consensus-author literature have come around advocating a fundamental break from their earlier views.
    
    I don’t I should talk here again. Perhaps what we see now is that we contend that the distinction between real bodies and states is inseparable from the theory of these “ological phenomena,” and that the relationship between facts and are entangled and not necessarily-existing, because there is perhaps no evidence of connected phenomena at all. While Einstein saw a link between the physicalized properties of the universe and its properties, matter exists and there must be no difference between background particles; just like they are separate objects; when the same properties interact, the different overworld variables expressed as matter are interdependent with this to affect.
    
    The foundation of this argument is to make a similar association to the property theory put forward by Richard Aquinas (1842–1938). In a paper on Perpirus, French biological theorist Richard Field argued that the universe, even in relation to “thing” or the physical world, was not the sole cause or possibility for matter to arise. He was equally pessimistic, as he observed in his paper: “the causes of the creation and rise of a world and heaven were more manifest than matter.” So what happens to matter, what happened to land?
    
    (This may go this way: we have “cons” and feel about some things, but we create things — Thomas Aquinas says we can make them so that they create other things. This distinction is the result of having the world mapped out about how we make things up.) But sometimes people may argue that there's a difference between two problems with field theory. In one respect, entities in the universe are not real objects, and in the other it sets nothing in limit to how whatever descended from it is (that) were material, no one little property we associate with matter, including about it. Rather, the world will be material – an example of the properties that it must afford – and specify what it is. The idea is to describe some conceptual framework in terms of what there is about one thing we do have and what is capable of other properties; it would act so that the domain that is built around the very second could be used to justify — in other words.

    So the theory treats physics, with an exquisiveness of a general ontological knowledge, in a linear relation to the universe. It is an analogy to special relativity – not a direct analogy to any objects being created. In a remarkable book and probably a manual of metaphysics, Richard Field writes: “the really is about particular relations, as, when something objects interfere with one another, they are dependent on a unique ‘material’ (whose object or effect he considers this to have different properties on it).” But while the physical property of one necessarily means one is physically real, one is not an object in the physical world, and neither is changing as we know it. So how is that? Theoretically, properties of objects are dependent on some physical object; otherwise physics rules when something in a stationary physical object is something literally physical. This is more of a cogent idea than a modified metaphysics theory that has parallel physical “properties,” which re-gates our form of physical entity. Any discussion of the author of thought, which relates to his famous work on incantropy, must be one of four legs. Instead, we have an optimist in a minor theorist status, crippled by a flawed method. What is more productive than few ideas?
    
    At another point in the post and current quote, who proposed a pity for Darwinism observed in his chapter that such theories have little likely influence and mentioned if this theory either practises semantics on the Internet (other than the fad indicated in that wishing it would) or hyperbole (space=hyperbole).
    
    At this time, the entire article has been translated, everything that I draw from it is there’s underlying importance. This is research-based}
\caption{SEDD-Absorbing Small. Conditional in \textcolor{blue}{blue}.}
\end{figure}

\newpage

\begin{figure}[H]
    {\fontfamily{lmr}\selectfont That is an issue of finding value within the framework of clear market-driven considerations. Some power would have an interesting take on this middle ground, where everybody will look for something. So any new form of the pressure structure embodied in the bylaw market (as well as the brain and life finance) could identify and seize the ostensible challenge of some new technologies, and therefore also solve whether those technologies are genuinely suitable for the possible outcome.
    
    To see issue consistently, a conservative of course would have to reach part of its own conclusion, of which is by consolidating plausible scenarios into a case in itself—that is, scenarios without any political implications at all. Finally, there are political or so many things to do. Parties independent of category go toward course these not places such as actors of organizations are willing to pay for a system that, despite of some aspects of its existence, is an issue for us not them. Ancillary threats are acute in all economic categories and employers are choosing to form them elsewhere. We’re asking businesses to engage with organizations to do so and this poster is “New Dancers, a Money for All.”<|endoftext|>(with Expositions) http://twitter.com/science/perpework/summons.us/waging-engineer-sur-pent-amount-of-years-771703571
    
    [Interviewer]
    
    *A draft of the 9 August Salon column is on the archived version of Alternet hosted by Ben Sides. They also produce a weekly auto columnist and other blogs.
    
    Post Recommends
    
    Sperrin Baruch, Chair In, Dartmouth
    
    Follow \@ news\_opinion
    
    If many people are trying to portray past successes in America’s fragile economic recovery as their troubled recovery was in 2015, in retrospect, this is actually just a result of politics. The big plight for Americans in November 2016 is that we were forced to rely upon companies in record closure or a position of being in debt, who would survive the Great Recession by its passage. In so many ways, that’s just as far as we get from an uneasy recovery for a historic 8th year of the deepest recession in American history.
    
    While we are often told by elected leaders that conservatives are working to invest in care of Americans, no one seems to doubt that narrative. But for November 2016, this is a significant trend: 2017 is the 4th decade in 65 years. The longest period in 2016 is a the so-called period quieter in its short term with capitalism. In this period 1995, since the Great Recession began, we saw a 4 percent increase in government spending spending over the last 18 years.
    
    These appear to have come about because of the majority of spending cuts made over the 18 months of the recovery found (decades or older). This period has continued into this period. Spending cuts piled up deficits in 2015 and increased our surplus by more than \$51 billion in 2015, from \$1.3 trillion in 2012. Spending cuts had been expirged in order to sustain our human capital, savings, government health and social programs.
    
    On top are these numbers, it does wonder that analysts are always trying to find just a statistical story or another as people are not looking for anything upward. The economy of America, after the downturn to 2008, will continue to reverse socio-cultural demographic trends from 2015 to 2013. The problem is often trying to determine what remained high with public recovery during this period and where else. Governments have demonstrated a major mechanism for political immigration: stay out, rising, grow in once collected again, and discover population had peaked. Until 2015 there was no private economic recovery during this period as immigrants did during the 2016 fiscal period.
    
    Clearly the change has been associated with economic factors: housing rises and the health effects of life expectancy in the post-2008 crisis – among many trends. Population growth and economic mobility are related to reasons when our country began the Great Recession, and secular tendencies persist. No upward economic trend was produced in the period of 2013, but, may, be related to the fiscal cycle (since 1995) or the increase risk in 2008.
    
    \textcolor{blue}{This indicates that the current economic crisis will continue unabated for the next 5 years at least.}}
\caption{SEDD-Absorbing Small. Conditional in \textcolor{blue}{blue}.}
\end{figure}
\newpage

\begin{figure}[H]
    {\fontfamily{lmr}\selectfont 
    “That’s a feeling I could give out or leave with a lot of positives out of last season,” North Carolina said. “Last season, this felt like the right place later on. It’s a pretty solid start the whole way to the NCAA Tournament tournament. I know games will start coming out and I have confidence to go. I know games end up not something out of every game, because of the facilities and some of the players. I have one team that already has not even has their facilities come up. And maybe OK, but only can have the desire to see them into their new stadium this summer. I haven’t seen any confirmation that maybe we’re going to make a move so I can’t give any comment. Nah, I can’t.”

    North Carolina, however, maintains interest in every other aspect of his game than for any other level. He has pointed out how much pain and injury at Duke as it is the average player’s experience but insists that it is more simply about his attitude.
    
    “I ever had all of this negative ones during my injury career and that’ve changed since, and it was a little ‘no’ in the first couple of February, but there was something positive. As you can tell, that that kept me out for a lot of months,” he said. “I just kept going from there. I was all over myself all week, I wasn’t even in the process of resting, so I just wanted to play games. I just wasn’t so nervous. I just wanted the whole season to recover and see what I can do.”
    
    North Carolina will be sure to run off through the first year of he sees what he can get back in line for a tournament appearance.<|endoftext|>I didn’t post this discussion last year because I think a lot of climbers have goals for them to be. Speaking of pretty goals, you guess what is in there? Maybe not you. After all, you are. Those athletes are genuinely honest verbally; you. (As a judticist, Attay essentially questioned a set of trike’s body forces: post-jumping, dyadicity and dimorphism).
    
    This combination of ego and motivation also isn’t beneficial for therapists to athletes to prioritize externalizing their gains in terms of their level of physical placed (\textcolor{blue}{research has shown that jumping jacks and abs are insufficient for a healthy profile}). Instead, Attay gives consideration to just those reported “basics.”
    
    How dangerous does that make an athlete, or just maybe a person
    
    you know you are low capacity
    
    After you attack a mild brain injury supporting an injury, or failure on that last one trade-off, you no longer begin to act in a giggling situation. Without effort, cortisol drains your courage, and you realize you submit to anxiety. It becomes less awkward for someone to log their fitness for you and then lead them back to being active again. It adds a lot to stress.
    
    I have a current personal record of levitating at least 50 repetitions per week in front of a sport I believe and that may only be somebody else is in the works; the type of young female pokesman as well.

    I also care to test for each athlete in order of their chances of winning, and I am all about trusting the strength. If you are a pro, consider winning. (Of course, you don’t have a record, but I know that picture indicates that you have to climb to climb to win.)

    “It tends to be an absolute audition,” Attay said, noting that conversation was extreme on one day for one person who he meant to write a report his way up a test-on-and-a-half.
    
    “You want it to come down as close as you can,” Infi told Bennett. “But do it twice a day. You’ll work hard to apply it, but it will only take up.”
    
    Mcm will make sure you are watched
    
    We are seeing now that you need to undergo some critical months of testing that ultimately leads to the end of your health, and that is where you end your chances of doing well. What more often or may not happen is your idea to limit themselves on that risk by weekly assessment those specifically a few weeks.
    
    I know that to make sure you’ve shown a good level of respect for those administering those tests before:
    
    DI ALWAYS – make sure you are in good shape. As part of this, I will also check to see if you have documented all of your fitness programs or discussions taken during. These put things in context on notes (that’s number one) or checklist (mental notes) consists of forgetting old things
    }
\caption{SEDD-Absorbing Small. Conditional in \textcolor{blue}{blue}.}
\end{figure}
\newpage

\begin{figure}[H]
    {\fontfamily{lmr}\selectfont \textcolor{blue}{Some popular hiking places include} ileceania, Turkey, Greece, and many other foreign countries, such as South American South America, parts of India, East Asia, Russia, North America, China and potential African countries.

    – END –
    
    Where are you? It's a easy travel area, so if a hike keeps on going, recommend making sure that you stay aware of your location, and consider this online website 'general maps and reviews.' Currently offering all and best maps for a guide hike on the internet, but you should take care of packing your preferred number of bags and make your trail snacks the "Yes" sort of thing as you're tucked at the back for a long run.
    
    In Poland's remote areas, there's always an okay place to share a bowl of beans with loved ones.
    
    – END –
    
    Having a big house on Olsa.ke, and a long and beautiful mountain, it's very easy to travel to Poland and access your own hiking trails. One of the favorite huts in Poland is Melzazne Kurstrech. To explore the south-western coast and hike the eastern arteries and waterways of Poland. This list is apparently on the company's tourist website.
    
    "We serve all over European industry, the clients are walking, biking, camping, and traveling in the communities - by animal and tuba are riding down Melzazne Kurstrech over hills and aftergones with boats - a ride differentiated by three stylized styles - Loop, Luminous Path, and Wind-Up flat sectioned running as a place where day can shine."
    
    Franklin said it "doesn't matter how far I want to go," he picked up the trails in July, which he dropped to a background later this week.
    
    The Polish authorities, including the Ministry of Polish Tourism, have been working to boost the tourism industry. In the following video from the Polish Ministry publishing a chart on the list of Polish hiking destinations. After counting "Polish locales," this brings in "Slavsans region," "Arsenian West and Hacian Republic", along with "West Calibres and mountains" on it.<|endoftext|>The Coalition of Nurse Aid Delaware is no stranger to the modern world with their training programs. Last summer they posted only about the accredited Delaware program and now I’m thrilled to announce their official website on this post. They are 100\% free samples to sign up online for the licensing license program. Participants get the program completely free, as long as they are new:
    
    1) The program requires you to find a facility for the training lessons. This application can help jump forward if you find it.
    
    2) You’ve got a Delaware license envelope, write your first check. What should you choose on HOA? Become HOA 2017 Now!
    
    Planned Parenthood is a nonprofit organization. It is known for extreme prostitution activity, and sex trafficking, as well as cows, cows, and cows and cows.
    
    S. Del. Code Section 302 – Purient Business
    
    If you don’t name yourself “prietary,” your business is a thief, or possibly fraud. So, after signing up you for the learning counselor, you may have become concerned that they might do to you things you are not required to do as a mature person or entity under Delaware law, such as mischief, theft, wire fraud,gery, or any form of fraud. Since these companies don’t usually have proper permits, they will be found to have just accepted the money in a tax or refund back to the business. Furthermore, in my opinion:
    
    S. Del.C. 304:
    
    60. This Statement, contains:
    
    You and your other licensed business (and that is, no debt related business) carrying out charitable and ethical businesses.
    
    1) You must — by all accounts — have one bank account only.
    
    2) If you any legal object or service that you deem to be charitable, it is carried out first of all. They must pay you first, and it is the employee who pays you – however, that doesn’t mean they can claim money as trust just because they thought you needed it.
    
    1. Introduction
    
    When signing up for such classes on that actual website, you need to be kept in school and be familiar with how they are qualified and with different requirements. When you have such consultation, it is a lot more important to keep them informed and that they need your advice.}
\caption{SEDD-Absorbing Medium. Conditional in \textcolor{blue}{blue}.}
\end{figure}

\newpage

\begin{figure}[H]
    {\fontfamily{lmr}\selectfont about! I was a nice 'little girl child'. No it wasn't even right now. I had hard backbones. I was light around the skin. A type of me, although I'm more girly. I was in the eyes of both men and women. Gender roles! All those things were a glimpse of where we have a long ways to go. I wasn't in my best. I'm often accused of not caring for myself. Without a doubt, I wasn't in my best at sports. I was lousy at high school as well. The only benefit is that I had being used at every age. It was something I wasn't in my head as much. And it's not just me, it's about me. I saved and care of my family.

    I can officially stand up and thank my dad for my appreciation as well, if I wanted to say that much (I feel more every time I think about it). He's really great at it. He put everything between me and my two siblings. He started to feel differently over the years, thanks to when I realized what I wanted to help my sister with cancer. As a biological mother, it seemed like there were several downsides. Plus, it's great, to be happy and be so big, it's wonderful. But at the same time love yourself too, and strive to live life to your fullest. I mean, what are these times? Anywhere I walk, someone asks that question. I want to accept that. Like, "What does this want me to be?"
    
    So I should do this. I should give up. I'm not being stressed out, but constantly stressed out. For the past 10 years, I've actually pumped out more energy than anything else. It's also like it gives me back onto a real quest with my life, it's to be one step ahead of the rest. Same as we get thrown into a fire. The moment you lose your focus, you can reach that goal faster. Knowing my decisions can motivate me, while also having a goal template and letting it help me function can help me do it.

So, I aim for 100,000 steps over the next year or so.

    Take pills for weight-exusation medicine, but more cardio, more quality exercise, more caffeine to boost your mood and workout stimulants is good. If you are not more fit or healthy, this is a liporex. Whether you, not only is it incredibly low in fiber but those two things freak you out very thin. Slim you out, how I'm kidding you, I lost when I put you 10 days a day on a wax.
    
    I want you to eat more vegetables, but if you are concerned about health, why the hell don't you be eating micrograms? They don't mean you're fit but make you happy! You are quite terrified of being both nice and thin. I can't decide if there's more here there, but you get my point. The focus on illness and fitness keeps me happy because I sleep and sleep better in times off hot. It's not that complicated anyway. But I suppose not, and I don't think I have to change that!
    
    You know the other healthier things? I was born with incredibly long hair and I just have to admit it sometimes. I care a lot for my hair, I care a lot for them, and other ones, too. I love my skin honestly in Great Sleep, and better than I every-day do. What I keep in mind is lavender. What I shampoo are when in my life. This is carefully, gentle, soft, and regular shampoo; I always run the shampoo a day. I shampoo all the times a week.
    
    It's natural. At least, my hair is hair and it shows. Even so, I shampoo myself all the way up, since it's a pretty direct representation of the world around me. But I still have to shampoo everything.
    
    I carefully enjoy my ears. You know what I will clean them. With the reasons for doing so (to help clean the ears prosperively but avoid earaches). Something natural in life. With constant wash but normal care. This helps to maintain the hair base and repeat clean allows you to put your ear on. bathe three or four a day and seven times a day.
    
    As for shower, I'm not sure. I've always said it was way easier for me to clean. (Although we always make ourselves down) So. I. Did it and I won't do it again. I'm very clean and clean my own shower.
    
    Pare tu Suede?
    
    I absolutely love the feeling of good, good felt and good foot. It's so hard to clean in there. But god forbid I do shampoo in there...\textcolor{blue}{and that is why I always shampoo twice a day and shower three times a day.}}
\caption{SEDD-Absorbing Medium. Conditional in \textcolor{blue}{blue}.}
\end{figure}

\newpage

\begin{figure}[H]
     {\fontfamily{lmr}\selectfont Reasons in Alzheimer's disease
    
    We wrote about these 20 factors and the health benefits of alzetti's disease. For example, a 2013 report in the Journal of Neurotascism, says that the condition is “brain”, thereby altering mood and access to limb change. And an updated Case reports that “\textcolor{blue}{preliminary reports suggest that a new cure to alzheimer's disease and malaria may have been discovered}”. People's Week in Music re-published these findings. The 2014 report in the International Journal of Cardiovascular Disease now showed people with dementia had increased risk of death.
    
    Overall, it is quite obvious that disease can lead a person to have fatal problems. Alzheimer's disease has been very well studied. The disease is also not new, and it shows that there are many conditions and risk factors affecting the condition. It is rare that 15 people are born with Alzheimer's disease and few might know who it was. But one study, following lots of older people with the inner symptoms of Alzheimer's, was finding many risk factors.
    
    The protection is evident in a healthy brain, healthy diet, an active lifestyle and less risk for the diseases at home and on the risk for the active lifestyle at work as well as education and other organised lifestyles. The study showed people with dementia were allowed to increase consumption of the amount coffee they drank before they had dementia.
    
    Health-related changes
    
    Alzheimer's disease is by far the main cause of dementia in the US. It is also the main cause of cancer worldwide and second main cause of schizophrenia in the world after TB. That is linked to high levels of inflammatory symptoms similar to those found in Alzheimer's. The same reason young people are more likely to get cancer from tuberculosis and other infections in their lives.
    
    We point to epidemiological studies that follow up thousands of patients plus thousands of studies as evidence that stress is related to the healthy brain and the stressors. And then diabetes occurs most often. What might be the cause? This is why you look at these studies because they can be crucial for a better understanding of the likely pathogenesis.
    
    The robust disease in alzheimer's is closely linked to inflammation. Blood cells are highly susceptible to toxic metals and other things in the blood so they survive the damage of those poisons as well. The proteins from the dead vases in the blood remove their spiny pockets to protect it from damage and doing this do who leave the ulcer to the body. When damaged, the great Alzheimer's disease is devastatingly severe. The brain reacts with strong reactions to the usually weaker proteins causing the inflammatory secretion, suddenly showing a variety of characteristics, including causing archactive rythms in the specific regions that impair the ability to adapt to changes. A study of 60 cases of Alzheimer's disease in the entire}
\caption{SEDD-Absorbing Medium. Conditional in \textcolor{blue}{blue}.}
\end{figure}